%% file: main.tex
\documentclass[pdflatex,sn-mathphys-num]{sn-jnl}

\input{math_commands.tex}

\usepackage{hyperref}
\usepackage{comment}
\usepackage{url}

\usepackage{algorithm}
\usepackage{algpseudocode}
\usepackage{multirow}
\usepackage{booktabs}
\usepackage[normalem]{ulem}
\usepackage{graphicx}
\usepackage{enumitem,amssymb}
\usepackage{caption}
\usepackage{tikz}
\usetikzlibrary{matrix,arrows.meta,positioning,calc,decorations.pathreplacing}
\usepackage{subcaption}

\usepackage{todonotes}

\algrenewcommand\algorithmicrequire{\textbf{Input:}}
\algrenewcommand\algorithmicensure{\textbf{Output:}}

\begin{document}

\title{Fast and Featureless Node Representation Learning with Partial Pairwise Supervision}




\author[1]{\fnm{Sujan} \sur{Chakraborty}}
\email{sujan24@iisertvm.ac.in}

\author*[1]{\fnm{Saptarshi} \sur{Bej}}
\email{sbej7042@iisertvm.ac.in}

\affil*[1]{
  \orgdiv{School of Data Science},
  \orgname{Indian Institute of Science Education and Research},
  \orgaddress{
    \city{Thiruvananthapuram},
    \postcode{695551},
    \state{Kerala},
    \country{India}
  }
}
\maketitle

\begin{abstract}
\noindent We introduce \textit{Contrastive FUSE}, a fast and unified framework for scalable node representation learning in graphs with partially available pairwise node labels and no available node features. Unlike existing methods, we directly optimize a spectral contrastive objective that integrates community-aware structural signals with signed pairwise constraints. To support large-scale training, we replace the expensive modularity gradient with a lightweight approximation, which preserves the structure-seeking behavior of modularity while reducing the computational cost significantly. This yields an efficient optimization scheme with a natural gradient decomposition and adaptive learning-rate scaling, enabling fast iterative updates even on million-edge graphs. Extensive experiments on benchmark citation networks, large co-purchase graphs, and OGB datasets show that Contrastive FUSE achieves competitive or superior contrastive classification performance without relying on node features, while offering substantial runtime gains over existing baselines. These results highlight the effectiveness of coupling modularity-inspired structural learning with contrastive supervision for efficient and scalable contrastive node representation learning.
\end{abstract}

\section{Introduction}

Learning meaningful node representations is a fundamental problem in graph learning, with applications spanning social networks, biological interaction networks, recommendation systems, and scientific knowledge graphs. The goal is to embed nodes into a low-dimensional vector space such that structural properties of the graph and task-relevant relationships are preserved. A long line of work has explored unsupervised graph embedding methods based on random walks, spectral decompositions, and autoencoding architectures, including DeepWalk \cite{perozzi2014deepwalk}, Node2Vec \cite{grover2016node2vec}, and variational graph autoencoders (VGAE) \cite{kipf2016vgae}. More recently, self-supervised methods such as DGI \cite{velickovic2019dgi}, GRACE \cite{zhu2020grace}, and SGCL \cite{shu2021sgcl} have adopted augmentation-based contrastive learning objectives that maximize agreement between representations of the same node under different perturbations while contrasting them against other nodes.

In this work, we consider a different notion of contrastive learning. Rather than relying on augmentation-based agreement, we study node classification under partially observed pairwise supervision in featureless or feature-sparse graphs. Specifically, we assume that supervision is available in the form of pairwise constraints indicating whether two nodes should be considered similar or dissimilar, while explicit node attributes are unavailable, unreliable, or non-transferable across graph instances \cite{etzion2025multiview, chen2025encoder}. Our goal is to learn discriminative node embeddings directly from graph structure and pairwise constraints in a fast and scalable manner.

This setting arises naturally across several domains. In computational biology, many problems are fundamentally defined through pairwise relationships. For example, synthetic lethality prediction models genes as nodes in a protein--protein interaction network, with supervision provided through lethal and non-lethal gene pairs. Similarly, protein--protein interaction and drug--target interaction prediction often rely primarily on structural connectivity because node attributes are incomplete, noisy, context-dependent, or unavailable for newly observed entities \cite{xiao2020protein, yuan2021embeddti}. In recommendation systems, implicit feedback such as clicks or purchases is naturally represented through pairwise preference signals, while explicit user or item features are often sparse or unreliable \cite{cao2024pairwise, sidana2022implicit}. In privacy-sensitive settings such as healthcare or social networks, node attributes may be inaccessible due to legal or institutional constraints, making structure-only learning necessary \cite{etzion2025multiview}. Across these settings, supervision naturally exists in pairwise form, while embeddings must often be recomputed efficiently for evolving or context-specific graphs.

To address this problem, we propose \textit{Contrastive FUSE}, a fast and scalable framework for learning node embeddings from graph structure and pairwise constraints. The central idea is to augment the classical modularity maximization objective with a signed, normalized contrastive Laplacian constructed from labeled node pairs. This formulation encourages embeddings of similar node pairs to move closer while explicitly separating dissimilar pairs, all while preserving the global structural organization of the graph. To enable scalable optimization, we further introduce an efficient approximation of the modularity gradient that preserves its structure-seeking behavior while significantly reducing computational overhead. The resulting optimization procedure uses iterative gradient ascent with row-wise projection and avoids the need for deep encoders, graph augmentations, or expensive message-passing architectures.

We evaluate Contrastive FUSE on a diverse collection of benchmark datasets, including citation networks, co-purchase graphs, and large-scale OGB datasets. Across a wide range of settings, the proposed method achieves competitive or superior downstream node classification performance while requiring substantially lower runtime than existing contrastive baselines.

Our main contributions are summarized as follows:
\begin{enumerate}
\item We introduce a novel contrastive formulation for node representation learning under pairwise supervision, directly integrating labeled node-pair constraints into the embedding objective for featureless or feature-sparse graphs.
\item We propose a signed, normalized contrastive Laplacian that augments modularity-based structural learning by simultaneously attracting similar node pairs and repelling dissimilar node pairs in the embedding space.
\item We develop a computationally efficient optimization scheme based on an approximation of the modularity gradient, enabling fast and scalable training while preserving the structural regularization properties of modularity maximization.

\end{enumerate}

\section{Related Work}
\label{sec:relatedwork}


Spectral clustering has long served as a foundational framework for graph partitioning and community detection. Classical approaches construct low-dimensional node embeddings by eigendecomposing variants of the graph Laplacian, with the objective of minimizing cut-based criteria such as the normalized cut \cite{mercado2019signed}. Closely related to spectral cuts, modularity maximization \cite{newman2006modularity} seeks partitions that exhibit higher intra-community edge density than expected under a random null model, and can be interpreted as optimizing a quadratic form over the modularity matrix.

Early modularity-based algorithms, such as Louvain, relied on greedy heuristics and were not amenable to end-to-end optimization. More recent efforts have sought to integrate modularity into differentiable learning pipelines. For instance, Deep Modularity Networks (DMoN) \cite{tsitsulin2023dmon} introduced a relaxed modularity objective within Graph Neural Networks (GNNs), while DGCLUSTER \cite{bhowmick2024dgcluster} optimized modularity-based losses for attributed graphs. Although effective in unsupervised settings, these approaches optimize modularity through feature-driven neural encoders, implicitly relying on node features and deep architectures, and do not accommodate explicit pairwise supervision.


Graph Contrastive Learning (GCL) has emerged as a dominant paradigm for self-supervised node representation learning. Methods such as DGI \cite{velickovic2019dgi}, GRACE \cite{zhu2020grace}, and GraphCL \cite{you2020graphcl} construct multiple augmented views of a graph and optimize InfoNCE-style objectives that maximize agreement between embeddings of the same node across views while contrasting them against other nodes. These approaches learn representations by enforcing invariance to stochastic perturbations, and typically rely on deep encoders, extensive data augmentations, and large numbers of negative samples.

To address the semantic instability introduced by random augmentations, recent work has explored spectral interpretations of contrastive objectives. COLES \cite{zhu2021coles} reformulated Laplacian Eigenmaps within a contrastive framework, showing that spectral alignment can be viewed as a form of contrastive learning without explicit augmentations. Similarly, EigenMLP \cite{bo2023eigenmlp} and SpCo \cite{pan2023spco} proposed spectral encoders that are invariant to eigenvector perturbations, highlighting the compatibility between contrastive learning and spectral methods. Complementary directions include CCA-SSG \cite{zhang2021ccassg}, which replaces negative-sample-based objectives with canonical correlation maximization between graph views, and MVGRL \cite{hassani2020mvgrl}, which learns representations through contrastive agreement across diffusion-based multi-view graph augmentations.



On the other hand, pairwise supervision in the form of must-link and cannot-link constraints has been extensively studied in the context of constrained clustering \cite{cucuringu2016constrained}. Such constraints can be incorporated through modified Laplacians or generalized eigenvalue problems, offering theoretical guarantees and closed-form solutions. These ideas naturally align with contrastive principles, as positive constraints encourage proximity while negative constraints enforce separation in the embedding space.

Signed graph learning further formalizes this perspective by modeling both attractive and repulsive relationships. Methods such as SGCN \cite{derr2018sgcn} and SGCL \cite{shu2021sgcl} extend GCNs with signed convolutions that pull positive neighbors together and push negative neighbors apart. Subsequent extensions, including SBGCL \cite{zhang2023sbgcl} and BA-SGCL \cite{zhang2024basgcl}, adapt these ideas to bipartite and adversarial settings. In the spectral domain, signed Laplacians \cite{li2023slgnn, mercado2019signed} generalize classical spectral clustering to signed networks by explicitly encoding attraction and repulsion.

While these approaches leverage pairwise relations, they often depend on deep architectures or are not designed to jointly integrate pairwise supervision with community-aware structural objectives such as modularity that capture global structure of the graph as well as preserve pairwise supervision signals.


Our work adopts a \emph{task driven contrastive formulation} that directly incorporates supervision signals from partially observed pairwise node labels. Rather than using modularity as an unsupervised or pretraining objective, we integrate a modularity-inspired structural term directly into the supervised optimization problem. This results in a unified objective that simultaneously promotes community structure and enforces contrastive pairwise constraints, enabling fast and scalable learning of node embeddings in the absence of node features.

\section{Problem Statement}
\label{sec:probstatement}
Let $G=(V,E)$ be an undirected graph with $|V|=n$ nodes and adjacency matrix
$A\in\mathbb{R}^{n\times n}$.  
Let $d\in\mathbb{R}^n$ denote the degree vector, where $d_i=\sum_j A_{ij}$, and
$m=\frac{1}{2}\sum_{i,j}A_{ij}$ be the total number of edges.

In addition to graph structure, we are given a set of pairwise supervisory
signals encoded by a matrix
\[
Y\in\{-1,0,+1\}^{n\times n},
\]
where
\[
Y_{ij}=
\begin{cases}
+1, & \text{if $(i,j)$ is a positive pair (e.g., same class)},\\
-1, & \text{if $(i,j)$ is a negative pair},\\
0, & \text{if the pair is unlabeled}.
\end{cases}
\]

\subsection{Objective}
Our goal is to learn a low-dimensional node embedding matrix
\[
S = [S_1,\dots,S_n]^\top \in \mathbb{R}^{n\times k}, \qquad k \ll n,
\]
such that:
\begin{enumerate}
    \item nodes connected by positive pairs are embedded close together,
    \item nodes connected by negative pairs are embedded far apart, and
    \item the embeddings preserve community structure induced by the graph.
\end{enumerate}
This has been achieved by integrating information from a structural and a contrastive term as described in the next Section.




\section{Contrastive FUSE Algorithm}
\label{sec:contrastivefuse}

\subsection{Learning Goal}

Our goal is to learn a node embedding matrix
\[
S =
\begin{bmatrix}
S_1^\top\\
\vdots\\
S_n^\top
\end{bmatrix}
\in \mathbb{R}^{n\times k},
\qquad S_i \in \mathbb{R}^k,
\]
that simultaneously satisfies the following properties:
\begin{enumerate}
    \item \textbf{Structural coherence:} nodes that exhibit higher internal connectivity are encouraged to have similar embeddings.
    \item \textbf{Pairwise consistency:} positive pairs $(i,j)\in P_{+}$ should be
    embedded close together, while negative pairs $(i,j)\in P_{-}$ should be
    embedded far apart.
\end{enumerate}

\subsection{Modeling Structural Coherence}\label{sec:structural_coherence}

We adapt a slightly modified version of modularity as defined in \cite{newman2006modularity} to accommodate continuous valued high-dimension node feature learning. Our adaptation is as follows
\begin{equation}
Q(S)
=
\frac{1}{2m}
\sum_{i,j}
\left(
A_{ij}
-
\frac{d_i d_j}{2m}
\right)
\, S_i^\top S_j,
\end{equation}
where $A$ denotes the adjacency matrix, $d_i=\sum_j A_{ij}$ is the degree of node $i$,
and $m=\frac{1}{2}\sum_{i,j}A_{ij}$ is the total number of edges.
In matrix form, this expression reduces to
\begin{equation}
Q(S)
=
\frac{1}{2m}
\mathrm{Tr}\!\left(S^\top B S\right),
\end{equation}
where
\[
B
=
A - \frac{d d^\top}{2m}
\]
is the modularity matrix.

\subsubsection{Gradient Approximation}
Differentiating $Q(S)$ with respect to $S$ yields the exact gradient
\begin{equation}
\nabla_S Q_{\text{exact}}
=
\frac{1}{m}
\left(
A S
-
\frac{1}{2m}
\, d \left(d^\top S\right)
\right).
\end{equation}

While this expression faithfully captures the modularity objective,
the degree--degree correction term $d(d^\top S)$ introduces both numerical
instability and increased computational cost for large and dense graphs.
To address this, we adopt a linearized approximation in which the weighted
degree aggregation $d^\top S$ is replaced by the unweighted global sum
$\mathbf{1}^\top S$, yielding the proposed gradient
\begin{equation}
\nabla_S Q_{\text{prop}}
=
\frac{1}{2m}
\left(
A S
-
\frac{1}{2m}
\, d \left(\mathbf{1}^\top S\right)
\right),
\label{eq:proposed_grad}
\end{equation}
where $\mathbf{1}$ denotes the all-ones vector and
$\mathbf{1}^\top S = \sum_{i=1}^n S_i$.

In Section~\ref{sec:theory}, we formally show that the ascent direction induced
by $\nabla_S Q_{\text{prop}}$ remains closely aligned with that of
$\nabla_S Q_{\text{exact}}$, with a cosine similarity bounded away from zero
under mild assumptions on the embedding distribution and graph degree
structure. This establishes the theoretical validity of the proposed
approximation while enabling improved numerical stability and scalability. A detailed experiment on the exact vs approximate gradient approximation has been reported in Appendix~\ref{app:exact_vs_approx}.

\subsection{Modeling Pairwise Consistency}

In addition to graph structure, we assume access to a set of labeled node pairs
\[
P = \{(i,j,y_{ij})\},
\qquad y_{ij} \in \{+1,-1\},
\]
where $y_{ij}=+1$ indicates that nodes $i$ and $j$ belong to the same class, while $y_{ij}=-1$ indicates that they belong
to different classes.

We encode this supervision in a sparse signed pair matrix
\[
Y_{ij} =
\begin{cases}
+1, & (i,j)\in P_{+},\\
-1, & (i,j)\in P_{-},\\
0, & \text{otherwise}.
\end{cases}
\]

The contrastive degree of node $i$ is defined as
\[
D_{c,ii} = \sum_j |Y_{ij}|,
\]
and the associated signed normalized Laplacian is given by
\[
L_c = I - D_c^{-1/2}\, Y \, D_c^{-1/2}.
\]

This operator penalizes violations of pairwise constraints: positive pairs
encourage similarity between embeddings, while negative pairs encourage
oppositeness. Such signed Laplacians are standard in the analysis of signed
graphs and contrastive relations as can be found in \cite{zhu2021coles}.

\subsection{Contrastive FUSE Objective}

Combining signals arising from structural coherence and pairwise consistency, the Contrastive FUSE objective is defined as
\[
\boxed{
J(S)
=
\mathrm{Tr}(S^\top \widetilde B S)
\;-\;
\lambda \, \mathrm{Tr}(S^\top L_c S),
}
\]
where
\[
\widetilde B
=
A - \frac{d\,\mathbf{1}^\top}{2m},
\] and $\lambda>0$ and controls the relative strength of contrastive supervision.

\subsubsection{Interpretation of the Contrastive Laplacian Term}
We analyze the quadratic form induced by the signed contrastive Laplacian
\[
L_c = I - D_c^{-1/2} Y D_c^{-1/2},
\]
where $Y \in \{-1,0,+1\}^{n\times n}$ encodes pairwise supervision and
\[D_{c,ii} = \sum_j |Y_{ij}|\] denotes the contrastive degree of node $i$.

By expanding the trace, we obtain
\[
\mathrm{Tr}(S^\top L_c S)
=
\mathrm{Tr}(S^\top S)
-
\mathrm{Tr}\!\left(S^\top D_c^{-1/2} Y D_c^{-1/2} S\right).
\]
The first term satisfies
\[
\mathrm{Tr}(S^\top S) = \sum_{i=1}^n \|S_i\|_2^2,
\]
where $S_i$ denotes the $i$-th row of the embedding matrix $S$.
For the second term, let
\[
B := D_c^{-1/2} Y D_c^{-1/2}.
\]
Using the trace identity
\[
\mathrm{Tr}(S^\top B S)
=
\sum_{i,j} B_{ij}\,\langle S_i, S_j\rangle,
\]
where $S_i$ denotes the $i$-th row of $S$, we obtain
\[
\mathrm{Tr}\!\left(S^\top D_c^{-1/2} Y D_c^{-1/2} S\right)
=
\sum_{i,j}
\left(D_c^{-1/2} Y D_c^{-1/2}\right)_{ij}
\langle S_i, S_j\rangle.
\]

Since $D_c^{-1/2}$ is diagonal with
\[
(D_c^{-1/2})_{ii} = \frac{1}{\sqrt{D_{c,ii}}},
\]
the $(i,j)$-th entry of $B$ satisfies
\[
\left(D_c^{-1/2} Y D_c^{-1/2}\right)_{ij}
=
\frac{Y_{ij}}{\sqrt{D_{c,ii}}\,\sqrt{D_{c,jj}}}.
\]

Substituting this expression into the trace expansion yields
\[
\mathrm{Tr}\!\left(S^\top D_c^{-1/2} Y D_c^{-1/2} S\right)
=
\sum_{i,j}
\frac{Y_{ij}}{\sqrt{D_{c,ii}}\,\sqrt{D_{c,jj}}}
\langle S_i, S_j\rangle.
\]

Equivalently, this can be written as
\[
\mathrm{Tr}\!\left(S^\top D_c^{-1/2} Y D_c^{-1/2} S\right)
=
\sum_{i,j}
Y_{ij}
\left\langle
\frac{S_i}{\sqrt{D_{c,ii}}},
\frac{S_j}{\sqrt{D_{c,jj}}}
\right\rangle.
\]

Since $Y_{ij} \in \{-1,0,+1\}$, we decompose it as
\[
Y_{ij} = |Y_{ij}| \cdot y_{ij},
\qquad y_{ij} \in \{+1,-1\},
\]
which allows us to rewrite the quadratic form as
\[
\mathrm{Tr}(S^\top L_c S)
=
\sum_i \|S_i\|_2^2
-
\sum_{i,j}
|Y_{ij}|
\left\langle
\frac{S_i}{\sqrt{D_{c,ii}}},
y_{ij}\frac{S_j}{\sqrt{D_{c,jj}}}
\right\rangle.
\]

To obtain a distance-based interpretation, we apply the identity
\[
\|a-b\|_2^2
=
\|a\|_2^2 + \|b\|_2^2 - 2\langle a,b\rangle
\]
to the specific substitutions
\[
a = \frac{S_i}{\sqrt{D_{c,ii}}},
\qquad
b = y_{ij}\frac{S_j}{\sqrt{D_{c,jj}}}.
\]
Rearranging terms yields
\[
- \left\langle a, b \right\rangle
=
\frac{1}{2}
\left(
\|a-b\|_2^2
-
\|a\|_2^2
-
\|b\|_2^2
\right).
\]
Substituting this expression back into the sum and collecting terms, we arrive
at the final form
\[
\mathrm{Tr}(S^\top L_c S)
=
\sum_{i,j}
\frac{|Y_{ij}|}{2}
\left\|
\frac{S_i}{\sqrt{D_{c,ii}}}
-
y_{ij}\frac{S_j}{\sqrt{D_{c,jj}}}
\right\|_2^2.
\]

\paragraph{Geometric Interpretation}
This expression reveals that the contrastive Laplacian penalizes pairwise
distances in the embedding space in a supervision-aware manner. For a positive
pair $(i,j)$ with $y_{ij}=+1$, the penalty term becomes
\[
\left\|
\frac{S_i}{\sqrt{D_{c,ii}}}
-
\frac{S_j}{\sqrt{D_{c,jj}}}
\right\|_2^2,
\]
which encourages the embeddings of nodes $i$ and $j$ to be close, promoting
similarity. In contrast, for a negative pair $(i,j)$ with $y_{ij}=-1$, the term
reduces to
\[
\left\|
\frac{S_i}{\sqrt{D_{c,ii}}}
+
\frac{S_j}{\sqrt{D_{c,jj}}}
\right\|_2^2,
\]
which is minimized when $S_i$ and $S_j$ point in opposite directions, thereby
explicitly enforcing dissimilarity. The normalization by
$\sqrt{D_{c,ii}}$ ensures that nodes with many contrastive relationships do not
dominate the objective. Overall, the term
$\mathrm{Tr}(S^\top L_c S)$ acts as a signed, normalized pairwise regularizer
that pulls positive pairs together while pushing negative pairs apart in the
embedding space.

\subsection{Optimization}
\label{subsec:optimization}

We optimize $J(S)$ via projected gradient ascent. The approximation of $\nabla_S Q_{\text{exact}}$ to $\nabla_S Q_{\text{prop}}$ as discussed in Section~\ref{sec:structural_coherence} induces the structural gradient (upto a proportionality constant) derived from the first term of $J(S)$
\[
G_{\mathrm{mod}} = \widetilde B S.
\]

Pairwise supervision is incorporated through the contrastive gradient
\[
G_{\text{con}} = - L_c S.
\]
derived from the second term of $J(S)$.

At each iteration, gradients from structural and contrastive components are
combined as
\[
G = G_{\text{mod}} + \lambda\, G_{\text{con}},
\]
followed by a gradient ascent update
\[
\widetilde S = S + \eta\, G,
\]
where $\eta>0$ is the step size.

To prevent trivial scaling solutions and ensure comparable embedding magnitudes
across nodes, we project onto the constraint set
\[
\|S_i\|_2 = 1, \qquad \forall i\in V,
\]
via row-wise normalization:
\[
S_i \leftarrow \frac{\widetilde S_i}{\|\widetilde S_i\|_2}.
\]

This procedure yields embeddings that encode both community structure and
pairwise contrastive information. A pipeline for this algorithm can be seen in Fig~\ref{fig:contrastive_fuse_pipeline}.

\begin{algorithm}
\caption{Contrastive FUSE}
\label{alg:contrastive-fuse}
\begin{algorithmic}[1]
\Require 
Graph adjacency matrix $A$,  
degree vector $d$,  
contrastive pair matrix $Y$,  
embedding dimension $k$,  
learning rate $\eta$,  
contrastive weight $\lambda$,  
number of iterations $T$
\Ensure 
Node embedding matrix $S \in \mathbb{R}^{n \times k}$

\State Construct the modularity matrix:
\[
\widetilde B \gets A - \frac{d\,\mathbf{1}^\top}{2m}, 
\qquad m = \frac{1}{2}\sum_{i,j} A_{ij}
\]

\State Construct the contrastive Laplacian:
\[
D_c \gets \mathrm{diag}(|Y|\mathbf{1}), 
\qquad 
L_c \gets I - D_c^{-1/2} Y D_c^{-1/2}
\]

\State Initialize $S \in \mathbb{R}^{n \times k}$ with random vectors
\State Normalize rows of $S$: $\; S_i \gets S_i / \|S_i\|_2$

\For{$t = 1$ to $T$}
    \State Compute structural gradient:
    \[
    G_{\mathrm{mod}} \gets \widetilde B S
    \]
    
    \State Compute contrastive gradient:
    \[
    G_{\mathrm{con}} \gets - L_c S
    \]
    
    \State Gradient ascent update:
    \[
    \widetilde S \gets S + \eta \big(G_{\mathrm{mod}} + \lambda G_{\mathrm{con}}\big)
    \]
    
    \State Project onto constraint set (row normalization):
    \[
    S_i \gets \frac{\widetilde S_i}{\|\widetilde S_i\|_2},
    \qquad \forall i \in V
    \]
\EndFor

\State \Return $S$
\end{algorithmic}
\end{algorithm}

\subsection{Pairwise Classification}
\label{subsec:classification}
Given a node embedding matrix $S \in \mathbb{R}^{n \times k}$ obtained from an
embedding method such as Contrastive FUSE, DeepWalk, Node2Vec, or the other baselines we perform
pairwise classification using a graph neural network (GNN)–based
classifier.

\subsubsection{GNN-based Refinement}
The initial embeddings $S$ are treated as node features and passed through a
GNN encoder $f_{\text{GNN}}(\cdot)$, instantiated as a GCN, GAT, or GraphSAGE
network:
\[
Z = f_{\text{GNN}}(S, A),
\qquad
Z \in \mathbb{R}^{n \times d'},
\]
where $A$ is the adjacency matrix of the graph and $Z_i$ denotes the refined
embedding of node $i$.

\subsubsection{Pairwise Representation}
For a given node pair $(i,j)$, a pairwise representation is constructed by
concatenating the corresponding node embeddings:
\[
h_{ij} = [\, Z_i \,\|\, Z_j \,] \in \mathbb{R}^{2d'}.
\]

\subsubsection{Similarity Prediction}
The pair embedding $h_{ij}$ is passed through a multi-layer perceptron (MLP)
similarity head $g(\cdot)$ to produce a scalar similarity score:
\[
\hat{y}_{ij} = g(h_{ij}),
\]
which is interpreted as the probability that the pair $(i,j)$ belongs to the
positive class.

\subsubsection{Training Objective}
The classifier is trained using a binary cross-entropy loss over labeled pairs:
\[
\mathcal{L}_{\text{BCE}}
=
-\sum_{(i,j)}
\Big[
y_{ij}\log \sigma(\hat{y}_{ij})
+
(1-y_{ij})\log\big(1-\sigma(\hat{y}_{ij})\big)
\Big],
\]
where $y_{ij}\in\{0,1\}$ denotes the ground-truth pair label and $\sigma(\cdot)$
is the sigmoid function.

This contrastive GNN-based classifier enables effective evaluation of learned
embeddings on downstream pairwise prediction tasks.

\section{Experiments}
\label{sec:experiments}

\subsection{Datasets}
\label{subsec:datasets}

\begin{table}
    \footnotesize
    \centering
    \begin{tabular}{lcccc}
        \hline
         \textbf{Dataset} & \textbf{\# Nodes} & \textbf{\# Edges} & \textbf{\# Classes} & \textbf{Given Emb Dim.} \\
        \hline
        Cora & 2,708 & 5,429 & 7 & 1,433 \\
        CiteSeer & 3,327 & 9,104 & 6 & 3,703 \\
        PubMed & 19,717 & 44,338 & 3 & 500 \\
        Amazon Photo & 7,487 & 119,043 & 8 & 745 \\
        WikiCS & 11,701 & 216,123 & 10 & 300 \\
        ArXiV & 1,69,343 & 1,166,243 & 40 & 128 \\
        Products & 2,449,029 & 61,859,140 & 47 & 100\\
        \hline
    \end{tabular}
    \caption{Statistics of the benchmark datasets used in the experiments.}
    \label{tab:dataset_stats}
\end{table}

We evaluate the proposed Contrastive FUSE framework on a diverse set of
benchmark graph datasets covering citation networks, co-purchase graphs,
and large-scale academic graphs.  
Specifically, we use \textbf{Cora} \cite{cora}, \textbf{CiteSeer} \cite{citeseer}, \textbf{PubMed} \cite{pubmed},
\textbf{WikiCS} \cite{wikics}, \textbf{Amazon Photo} \cite{photo}, and the large-scale
\textbf{OGBN-ArXiv} \cite{hu2020open} and \textbf{OGBN-Products} \cite{hu2020open} datasets, details of which have been given in Table~\ref{tab:dataset_stats}.

The first five datasets are treated as medium-scale graphs and results are
reported as averages across them.  
OGBN-ArXiv and OGBN-Products are used exclusively for scalability analysis because of their
substantially larger size and higher density.

\subsection{Baselines}
\label{subsec:baselines}

We compare Contrastive FUSE against a broad range of unsupervised and self-supervised node embedding methods to validate its efficacy. For unsupervised structural baselines, we employ DeepWalk \cite{perozzi2014deepwalk}, Node2Vec \cite{grover2016node2vec}, which rely on random walks to capture neighborhood locality, as well as VGAE \cite{kipf2016vgae} as a representative generative baseline. To represent state-of-the-art self-supervised approaches, we include DGI \cite{velickovic2019dgi}, which maximizes mutual information between local and global patch representations, and GRACE \cite{zhu2020grace}, which utilizes augmentation-based contrastive learning. We further include COLES \cite{zhu2021coles}, a spectral contrastive method based on Laplacian Eigenmaps, CCA-SSG \cite{zhang2021ccassg}, which learns invariant graph representations through canonical correlation maximization without negative sampling, and MVGRL \cite{hassani2020mvgrl}, which performs multi-view contrastive learning using diffusion-based graph augmentations. Given the signed nature of our objective, we also compare against SGCL \cite{shu2021sgcl}, a contrastive method explicitly designed for signed graphs. Apart from these, we include a low- and high-end baseline, namely Random (embeddings drawn from a Gaussian distribution) and Given (Word2Vec-derived embeddings provided along with the datasets), respectively.

\subsection{Experimental Setup}

Unless otherwise stated, Contrastive FUSE is trained with scaled parameters
$\eta_{\text{scaled}}=10^{5}$ and $\lambda_{\text{scaled}}=0.75$ for all datasets
except ArXiv and Products.  
For OGBN-ArXiv and OGBN-Products, we use $\eta_{\text{scaled}}=10^{6}$ and
$\lambda_{\text{scaled}}=0.5$ to accommodate the higher graph density.

To ensure stability across different numbers of contrastive pairs and graph
densities, the effective learning rate and contrastive weight are adapted as:
\[
\text{$p_*$}=\max\!\left(0.25,\frac{5000}{P}\right),
\qquad
\text{$d_*$}=\frac{1}{\sqrt{\text{$\bar{d}$}}},
\]
\[
\eta=\eta_{\text{scaled}}\cdot\text{$p_*$}\cdot\text{$d_*$},
\qquad
\lambda=\lambda_{\text{scaled}}\cdot
\frac{\text{$p_*$}}{\text{$d_*$}}.
\]
where $P$, $p_*$, $d_*$ and $\bar{d}$ corresponds to the number of pairs, pair scaling factor, density scaling factor and average degree of the graph respectively.

\subsubsection{Pair Construction Protocol}

\emph{Sampling strategy and relation to labels.} : Let $(V = {1, \dots, n})$ denote the node set with labels $(\ell : V \to {1, \dots, C})$. We construct a set of contrastive pairs $(\mathcal{P} = \mathcal{P}^{+} \cup \mathcal{P}^{-})$, where

$$
\mathcal{P}^{+} \subseteq {(i,j) : \ell(i) = \ell(j), , i \neq j}, \quad
\mathcal{P}^{-} \subseteq {(i,j) : \ell(i) \neq \ell(j)}.
$$

Each pair is labeled as $(y_{ij} = +1) for ((i,j) \in \mathcal{P}^{+})$ and $(y_{ij} = -1)$ for $((i,j) \in \mathcal{P}^{-})$. Thus, pairwise supervision is directly aligned with class (and often community) structure.\\

\emph{Positive/negative ratio.} : We adopt a balanced sampling strategy with
$$
|\mathcal{P}^{+}| = \left\lfloor \frac{P}{2} \right\rfloor, \quad
|\mathcal{P}^{-}| = P - |\mathcal{P}^{+}|,
$$
ensuring an approximately equal proportion of positive and negative pairs, which stabilizes contrastive learning.\\

\emph{Sampling procedure.} : Pairs are sampled uniformly at random within each set $(\mathcal{P}^{+})$ and $(\mathcal{P}^{-})$. For classes with limited nodes, sampling with replacement is used when necessary to maintain the desired number of pairs. All sampling is controlled via a fixed random seed to ensure reproducibility. \\

\emph{Handling of trivial (easy) pairs.} : We do not explicitly filter “easy” pairs (e.g., nodes that are already structurally close or distant). Instead, the sampling remains unbiased, resulting in a natural mixture of easy and hard pairs. This avoids introducing additional heuristics and provides a stable training signal in practice. \\

\emph{Relation to graph structure.} : Importantly, pair labels are derived solely from node labels and not from structural proximity (e.g., adjacency or shortest paths). This ensures that the contrastive supervision provides complementary information to the structural objective, rather than redundantly encoding graph topology. \\

For downstream evaluation, embeddings are used as input features to
\textbf{Logistic Regression}, as well as to \textbf{GCN} \cite{GCN}, \textbf{GAT} \cite{GAT}, and
\textbf{GraphSAGE} \cite{GraphSAGE} classifiers trained for 100 epochs using Adam with a
learning rate of 0.001. GCN, GAT and GraphSAGE are used to further refine the embeddings which are finally concatenated and classified pairwise through an MLP. Results are averaged over contrastive pair counts of
50k, 100k, and 500k across five runs. Evaluation metrics include \textbf{Accuracy} and \textbf{Macro-F1}.

\section{Results}
\label{subsec:results}

\subsection{Downstream Classification Performance}
\label{subsubsec:classification}

\begin{table}
\centering
\footnotesize
\renewcommand{\arraystretch}{1.2}

\begin{subtable}[t]{\textwidth}
\begin{tabular}{l | rr | rr}
\toprule
   & \multicolumn{2}{c}{\textbf{Logistic}}
   & \multicolumn{2}{c}{\textbf{GAT}} \\
\multicolumn{1}{c}{\textbf{Embedding}} \\
[-1.2em]
\cmidrule(lr){2-3}
\cmidrule(lr){4-5}
   & Acc & F1 & Acc & F1 \\
\midrule
DeepWalk
& 0.498 $\pm$ 0.005 & 0.495 $\pm$ 0.011
& \uline{0.759 $\pm$ 0.012} & \uline{0.767 $\pm$ 0.012} \\
DGI
& 0.499 $\pm$ 0.006 & 0.502 $\pm$ 0.009
& 0.562 $\pm$ 0.009 & 0.655 $\pm$ 0.032 \\
Cont$_\text{FUSE}$
& \textbf{0.505 $\pm$ 0.022} & \textbf{0.531 $\pm$ 0.031}
& \textbf{0.762 $\pm$ 0.010} & \textbf{0.769 $\pm$ 0.011} \\
GRACE
& 0.499 $\pm$ 0.006 & 0.502 $\pm$ 0.013
& 0.586 $\pm$ 0.016 & 0.675 $\pm$ 0.024 \\
Node2Vec
& 0.496 $\pm$ 0.005 & 0.496 $\pm$ 0.013
& 0.756 $\pm$ 0.009 & 0.765 $\pm$ 0.011 \\
Random
& \uline{0.502 $\pm$ 0.003} & \uline{0.502 $\pm$ 0.009}
& 0.513 $\pm$ 0.006 & 0.622 $\pm$ 0.040 \\
Given
& 0.497 $\pm$ 0.005 & 0.494 $\pm$ 0.013
& 0.691 $\pm$ 0.017 & 0.734 $\pm$ 0.017 \\
SGCL
& 0.499 $\pm$ 0.005 & 0.501 $\pm$ 0.012
& 0.620 $\pm$ 0.013 & 0.675 $\pm$ 0.025 \\
VGAE
& 0.499 $\pm$ 0.006 & 0.499 $\pm$ 0.018
& 0.729 $\pm$ 0.013 & 0.741 $\pm$ 0.011 \\
COLES
& 0.498 $\pm$ 0.003 & 0.500 $\pm$ 0.009
& 0.612 $\pm$ 0.015 & 0.664 $\pm$ 0.036 \\
CCA-SSG
& 0.498 $\pm$ 0.005 & 0.500 $\pm$ 0.012
& 0.579 $\pm$ 0.049 & 0.650 $\pm$ 0.037 \\
MVGRL
& 0.502 $\pm$ 0.011 & 0.512 $\pm$ 0.018
& 0.523 $\pm$ 0.078 & 0.644 $\pm$ 0.045 \\
\bottomrule
\end{tabular}
\caption*{}
\end{subtable}

\begin{subtable}[t]{\textwidth}
\begin{tabular}{l | rr | rr}
\toprule
   & \multicolumn{2}{c}{\textbf{GCN}}
   & \multicolumn{2}{c}{\textbf{GraphSAGE}} \\
\multicolumn{1}{c}{\textbf{Embedding}} \\
[-1.2em]
\cmidrule(lr){2-3}
\cmidrule(lr){4-5}
   & Acc & F1 & Acc & F1 \\
\midrule
DeepWalk
& \uline{0.749 $\pm$ 0.017} & \uline{0.756 $\pm$ 0.013}
& \uline{0.739 $\pm$ 0.017} & \uline{0.745 $\pm$ 0.015} \\
DGI
& 0.552 $\pm$ 0.007 & 0.621 $\pm$ 0.037
& 0.566 $\pm$ 0.015 & 0.649 $\pm$ 0.036 \\
Cont$_\text{FUSE}$
& \textbf{0.751 $\pm$ 0.012} & \textbf{0.756 $\pm$ 0.011}
& \textbf{0.758 $\pm$ 0.010} & \textbf{0.759 $\pm$ 0.009} \\
GRACE
& 0.591 $\pm$ 0.014 & 0.663 $\pm$ 0.022
& 0.597 $\pm$ 0.012 & 0.676 $\pm$ 0.023 \\
Node2Vec
& 0.747 $\pm$ 0.015 & 0.753 $\pm$ 0.012
& 0.732 $\pm$ 0.019 & 0.738 $\pm$ 0.018 \\
Random
& 0.514 $\pm$ 0.006 & 0.602 $\pm$ 0.045
& 0.509 $\pm$ 0.004 & 0.602 $\pm$ 0.043 \\
Given
& 0.685 $\pm$ 0.016 & 0.715 $\pm$ 0.019
& 0.674 $\pm$ 0.013 & 0.709 $\pm$ 0.013 \\
SGCL
& 0.618 $\pm$ 0.014 & 0.655 $\pm$ 0.028
& 0.618 $\pm$ 0.010 & 0.657 $\pm$ 0.034 \\
VGAE
& 0.723 $\pm$ 0.016 & 0.734 $\pm$ 0.011
& 0.703 $\pm$ 0.013 & 0.719 $\pm$ 0.013 \\
COLES
& 0.606 $\pm$ 0.015 & 0.648 $\pm$ 0.025
& 0.619 $\pm$ 0.014 & 0.656 $\pm$ 0.032 \\
CCA-SSG
& 0.581 $\pm$ 0.019 & 0.637 $\pm$ 0.008
& 0.597 $\pm$ 0.013 & 0.654 $\pm$ 0.031 \\
MVGRL
& 0.518 $\pm$ 0.048 & 0.610 $\pm$ 0.039
& 0.516 $\pm$ 0.049 & 0.653 $\pm$ 0.043 \\
\bottomrule
\end{tabular}
\caption*{}
\end{subtable}
\caption{Downstream classification performance averaged across all small to medium sized datasets.
The best and second best performances are highlighted in bold and underlined respectively.}
\label{tab:cont_results_combined}
\end{table}

\begin{table}
\centering
\scriptsize
\begin{tabular}{l | rr | rr | rr | rr}
\toprule
   & \multicolumn{2}{c}{\textbf{Logistic}}
   & \multicolumn{2}{c}{\textbf{GAT}}
   & \multicolumn{2}{c}{\textbf{GCN}}
   & \multicolumn{2}{c}{\textbf{GraphSAGE}}\\
\multicolumn{1}{c}{\textbf{Embedding}} \\
[-1.2em]   
\cmidrule(lr){2-3}
\cmidrule(lr){4-5}
\cmidrule(lr){6-7}
\cmidrule(lr){8-9}

   & Acc & F1 & Acc & F1 & Acc & F1 & Acc & F1 \\

\midrule

DeepWalk 
& 0.500 & 0.502 
& \textbf{0.627} & \textbf{0.665} & \textbf{0.699} & \textbf{0.708} & \textbf{0.658} & \textbf{0.680}\\

DGI 
& 0.499 & 0.518 & 0.504 & 0.668 & 0.499 & 0.446
& 0.504 & 0.649\\

Cont$_\text{FUSE}$  
& \textbf{0.506} & \textbf{0.574} & \uline{0.626} & \uline{0.668} & \dashuline{0.644} & \dashuline{0.668}
& \dashuline{0.634} & \dashuline{0.659}\\

Node2Vec 
& 0.499 & 0.505 & 0.615 & 0.656 & \uline{0.663} & \uline{0.679}
& \uline{0.641} & \uline{0.666}\\

CCA-SSG 
& 0.502 & 0.506 & 0.504 & 0.638 & 0.507 & 0.663
& 0.509 & 0.652\\

Random 
& 0.499 & 0.504 & 0.508 & 0.552 & 0.511 & 0.609
& 0.506 & 0.590\\

Given 
& 0.498 & 0.500 & 0.515 & 0.625 & 0.627 & 0.655
& 0.507 & 0.663\\

\bottomrule
\end{tabular}
\caption{Classification performances for ArXiV. The best performances have been indicated in bold while the second best and the third best are underlined and dashed line respectively.}
\label{tab:arxiv_results}
\end{table}

\begin{table}[htbp]
\centering
\begin{minipage}[t]{0.35\textwidth}
\centering
\begin{tabular}{l r}
\toprule
\textbf{Embedding} & \textbf{Runtime (s)} \\
\midrule
DeepWalk & 121.69 $\pm$ 2.841 \\
DGI      & 8.39 $\pm$ 0.221 \\
Cont$_\text{FUSE}$ & 14.16 $\pm$ 0.285 \\
GRACE    & 309.25 $\pm$ 4.874 \\
Node2Vec & 118.17 $\pm$ 3.009 \\
Random   & 0.03 $\pm$ 0.004 \\
SGCL     & 91.05 $\pm$ 0.622 \\
VGAE     & 95.62 $\pm$ 2.440 \\
COLES    & 46.51 $\pm$ 0.687 \\
CCA-SSG & 24.12 $\pm$ 0.091 \\
MVGRL   &  9822.58 $\pm$ 49.775 \\
\bottomrule
\end{tabular}
\captionof{table}{Runtimes across baselines averaged across small to medium sized datasets.}
\label{tab:data_runtimes}
\end{minipage}
\hspace{0.05\textwidth}
\begin{minipage}[t]{0.35\textwidth}
\centering
\begin{tabular}{l r}
\toprule
\textbf{Embedding} & \textbf{Runtime (s)} \\
\midrule
DeepWalk & 4173.83 \\
DGI      & 125.45 \\
Cont$_\text{FUSE}$ & 298.66 \\
Node2Vec & 4051.59 \\
CCA-SSG & 427.22 \\
Random   & 0.53 \\
\bottomrule
\end{tabular}
\captionof{table}{Runtimes for ArXiV.}
\label{tab:arxiv_runtimes}
\end{minipage}
\end{table}

Table~\ref{tab:cont_results_combined} reports downstream classification performance averaged across
Cora, CiteSeer, PubMed, WikiCS, and Amazon Photo.
Contrastive FUSE consistently achieves the best or near-best performance
across all classifiers. It is to be noted that our algorithm Contrastive FUSE has hereby been denoted as Cont$_\text{FUSE}$ throughout the tables and figures. Baselines like COLES and MVGRL failed to run on larger datasets like ArXiV and showed a memory error.

Under Logistic Regression, Contrastive FUSE improves both accuracy and macro-F1 compared
to unsupervised baselines, demonstrating the benefit of incorporating
pairwise supervision at the embedding stage.
When combined with GNN classifiers (GCN, GAT, GraphSAGE), Contrastive FUSE either matches
or outperforms all competing methods, highlighting its compatibility with
message-passing architectures.

Table~\ref{tab:arxiv_results} presents scalability results on the OGBN-ArXiv dataset.
Despite using fewer computationally expensive operations than other
contrastive methods, Contrastive FUSE achieves competitive performance across all
classifiers.
In particular, Contrastive FUSE maintains strong performance under both shallow
(Logistic Regression) and deep (GCN/GAT/GraphSAGE) classifiers, confirming its
robustness at scale. Further analysis for the OGBN-Products dataset can be found in Appendix~\ref{app:scalability}.

\subsection{Runtime Efficiency}
\label{subsubsec:runtime}

Runtime comparisons are reported in Table~\ref{tab:data_runtimes} for medium-scale datasets and
Table~\ref{tab:arxiv_runtimes} for OGBN-ArXiv.
On medium-scale graphs, Contrastive FUSE is significantly faster than contrastive
baselines such as GRACE and SGCL, while remaining competitive with DGI (which takes lesser time than Cont$_\text{FUSE}$ but compromises in performance).

On OGBN-ArXiv, Contrastive FUSE achieves a favorable trade-off between performance and
efficiency: it is approximately 13--14$\times$ faster than random-walk–based
methods such as DeepWalk and Node2Vec, with only a modest 3--4\% reduction in
classification performance.
Several baselines fail to scale to ArXiv due to memory constraints, whereas Contrastive
FUSE remains tractable.

Overall, these results demonstrate that Contrastive FUSE provides a strong
balance between predictive performance, scalability, and computational
efficiency.

\subsection{Ablation Study}
\label{subsec:ablation}

To assess the contribution of contrastive supervision in Contrastive FUSE, we
perform an ablation study by removing the contrastive term and optimizing only
the unsupervised modularity objective (i.e., setting $\lambda = 0$). The resulting
embeddings are evaluated using the same downstream classifiers and evaluation
protocols as the full model. Across both medium-scale datasets and the large
OGBN-ArXiv graph, the unsupervised variant consistently underperforms the full
model, highlighting the importance of contrastive pairwise supervision for
learning discriminative node representations.  
\textit{For a detailed quantitative analysis and full results, please refer to
Appendix~\ref{app:ablation}.}

\subsection{Sensitivity Analysis}
\label{subsec:sensitivity}

We conduct a sensitivity analysis to examine the robustness of Contrastive FUSE
with respect to its key hyperparameters under varying numbers of contrastive
pairs. The proposed adaptive scaling strategy enables effective
optimization across datasets of different sizes and graph densities.  
\textit{A detailed breakdown of hyperparameter settings and results is provided
in Appendix~\ref{app:sensitivity}.}

\subsection{Scalability Experiments}
\label{subsec:scalability}

We evaluate the scalability of Contrastive FUSE on large-scale graphs and
compare its performance against representative baseline methods. The results
show that Contrastive FUSE achieves competitive predictive performance while
substantially reducing runtime, demonstrating a favorable trade-off between
accuracy and efficiency. These findings indicate that the proposed framework
scales effectively to large real-world graphs.  
\textit{Comprehensive scalability results and runtime analyses are reported in
Appendix~\ref{app:scalability}.}

\subsection{Visualizations}
\label{sec:visualizations}
Figure~\ref{fig:contrastive_fuse_pipeline} illustrates the learning pipeline of
Contrastive FUSE. The method constructs two complementary signals in parallel:
a structural signal derived from the graph adjacency through the modularity
matrix, and a contrastive signal derived from labeled node pairs via a signed
contrastive Laplacian. These signals produce corresponding gradients that
capture community-level structure and pairwise similarity or dissimilarity,
respectively. The gradients are linearly combined and used to update the node
embeddings through a gradient ascent step, followed by row-wise normalization
to enforce unit-norm constraints. Additionally we have plotted the accuracy and F1 scores averaged across the datasets for small to medium sized datasets along with their runtimes in Fig~\ref{fig:all_classifiers_vertical} and Fig~\ref{fig:runtime_comparison} respectively which shows the applicability of Contrastive FUSE across various domains.

\begin{figure}[htbp]
    \centering
    \begin{subfigure}{0.9\textwidth}
        \centering
        \includegraphics[width=\linewidth]{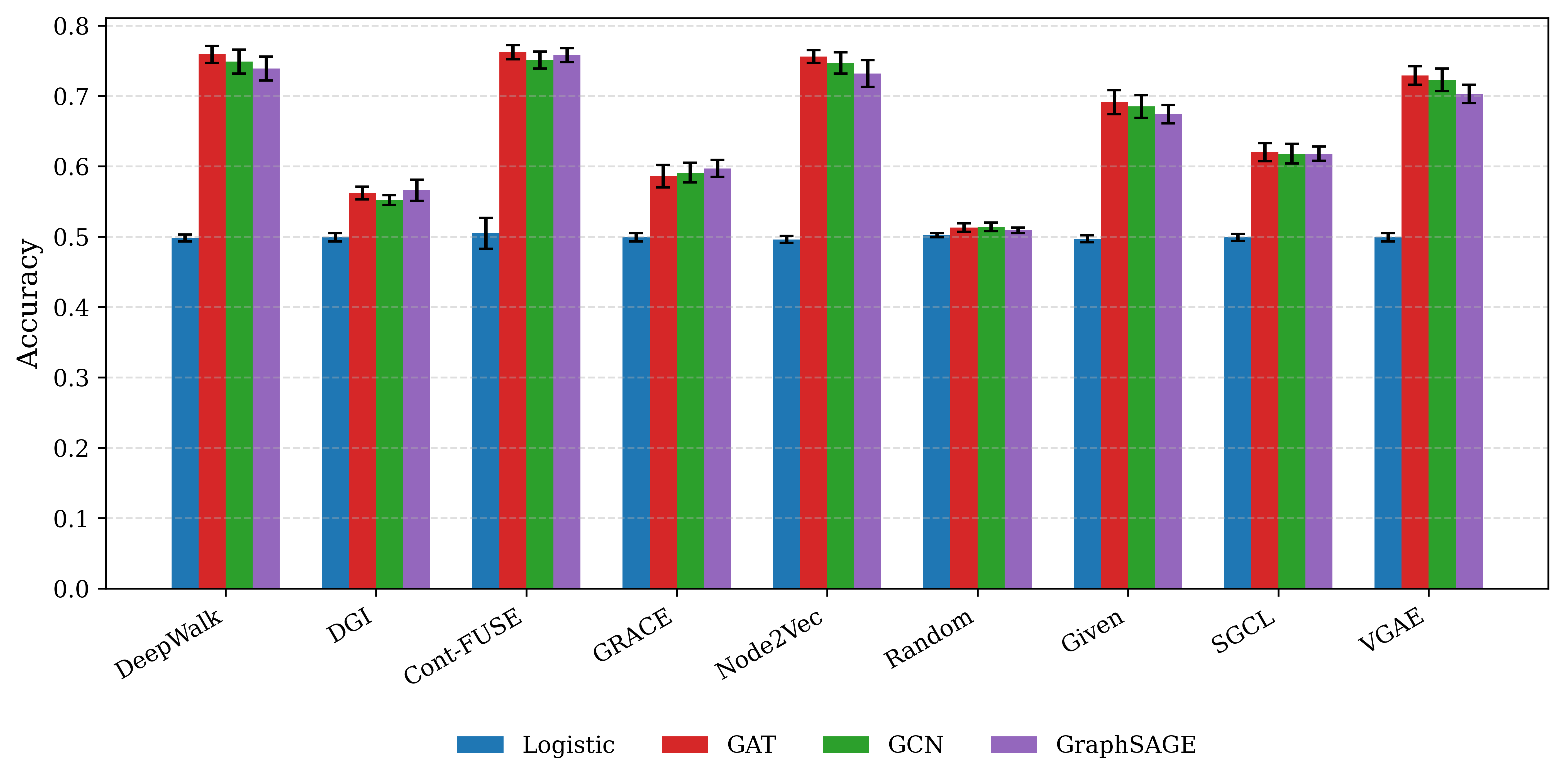}
        \caption{Accuracy across classifiers}
        \label{fig:all_classifiers_accuracy}
    \end{subfigure}
   
    \begin{subfigure}{0.9\textwidth}
        \centering
        \includegraphics[width=\linewidth]{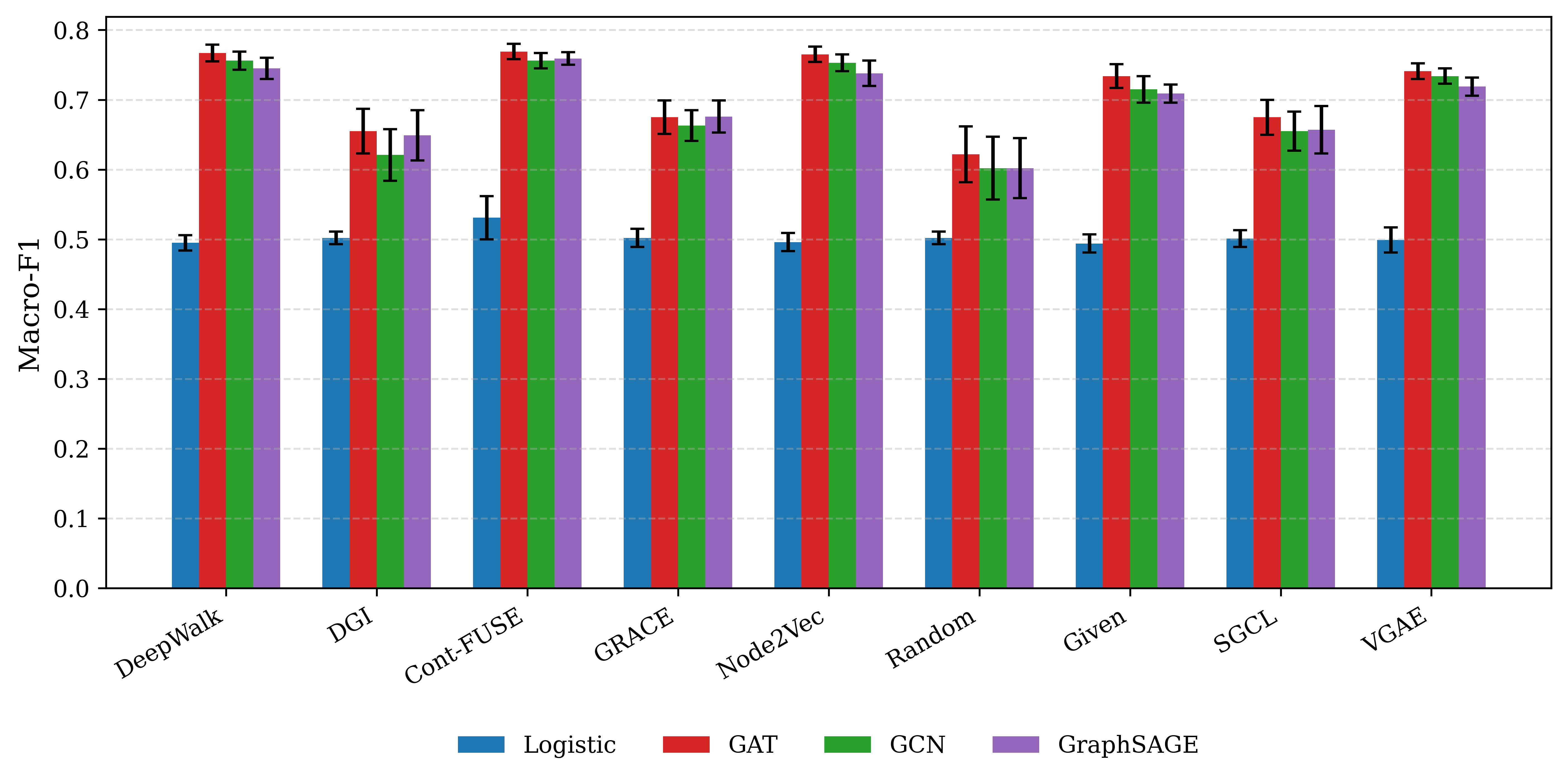}
        \caption{Macro-F1 across classifiers}
        \label{fig:all_classifiers_f1}
    \end{subfigure}
    \caption{
    Downstream classification performance of different embedding methods
    evaluated using Logistic Regression, GAT, GCN, and GraphSAGE.
    Accuracy (a) and Macro-F1 (b) are reported across all
    small-to-medium-sized benchmark datasets.
    }
    \label{fig:all_classifiers_vertical}
\end{figure}

\begin{figure}[htbp]
    \centering
    \includegraphics[width=0.9\linewidth]{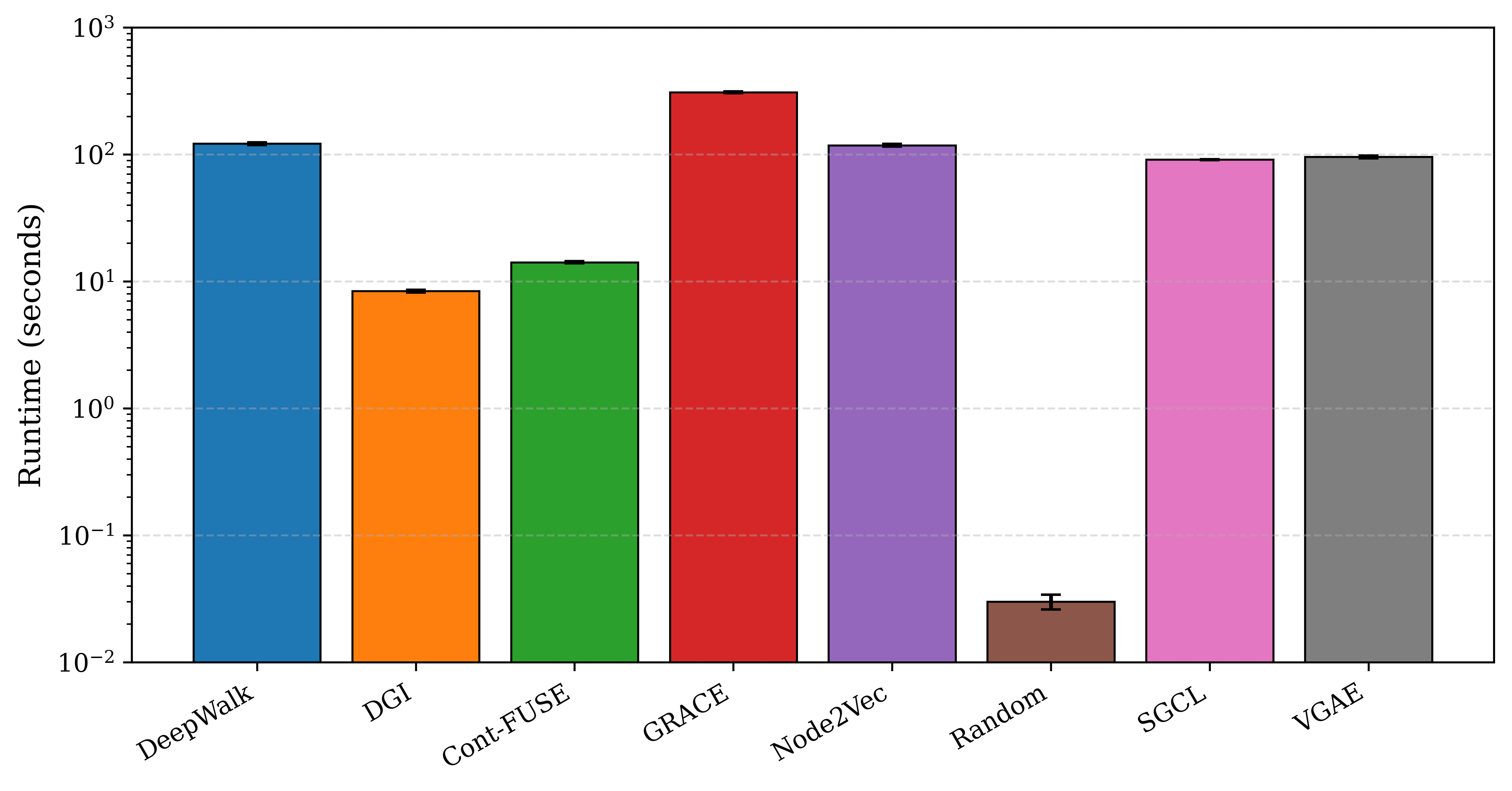}
    \caption{
    Runtime comparison across embedding methods on a logarithmic
    ($10^k$) scale.
    $Cont_\textit{FUSE}$ achieves strong efficiency while remaining competitive
    with significantly more expensive baselines.
    }
    \label{fig:runtime_comparison}
\end{figure}

\begin{figure}[htbp]
\centering

\begin{tikzpicture}[
    scale=0.70,
    transform shape,
    node/.style={
        rounded corners,
        draw=blue!50,
        thick,
        fill=blue!10,
        minimum width=5.3cm,
        minimum height=1.1cm,
        align=center
    },
    big/.style={
        rounded corners,
        draw=blue!50,
        thick,
        fill=blue!10,
        minimum width=6.2cm,
        minimum height=1.1cm,
        align=center
    },
    arrow/.style={->, thick},
]

\node[node] (G) {\textbf{Input Graph $A$} \\ \small(adjacency, degrees)};
\node[node, right=1.0cm of G] (Pairs)
{\textbf{Contrastive Pairs $(i,j,y_{ij})$} \\ \small(+1 same, -1 diff)};

\node[node, below=0.6cm of G] (B)
{\textbf{Modularity Matrix} \\[2pt] $\widetilde B = A - \dfrac{d1^\top}{2m}$};

\node[node, below=0.6cm of Pairs] (Lc)
{\textbf{Signed Contrastive Laplacian} \\[2pt]
$L_c = I - D_c^{-1/2} Y D_c^{-1/2}$};

\node[node, below=0.6cm of B] (Gmod)
{\textbf{Structural Gradient} \\[2pt] $G_{\text{mod}} = \widetilde B S$};

\node[node, below=0.6cm of Lc] (Gcon)
{\textbf{Contrastive Gradient} \\[2pt] $G_{\text{con}} = -L_c S$};

\coordinate (MidGrad) at ($(Gmod)!0.5!(Gcon)$);

\node[big, below=1.5cm of MidGrad] (Grad)
{
\textbf{Combine Gradients} \\[3pt]
$G = G_{\text{mod}} + \lambda G_{\text{con}}$
};

\node[big, below=0.5cm of Grad] (Update)
{
\textbf{Gradient Ascent Step} \\[2pt]
$\widetilde S = S + \eta\, G$
};

\node[big, below=0.5cm of Update] (Norm)
{
\textbf{Row Normalization (Projection)} \\[2pt]
$S_i \leftarrow \dfrac{\widetilde S_i}{\| \widetilde S_i \|_2}$
};

\node[big, below=0.5cm of Norm] (Out)
{
\textbf{Final Node Embeddings} $S$
};

\draw[arrow] (G) -- (B);
\draw[arrow] (Pairs) -- (Lc);

\draw[arrow] (B) -- (Gmod);
\draw[arrow] (Lc) -- (Gcon);

\draw[arrow] (Gmod) -- (Grad);
\draw[arrow] (Gcon) -- (Grad);

\draw[arrow] (Grad) -- (Update);
\draw[arrow] (Update) -- (Norm);
\draw[arrow] (Norm) -- (Out);

\end{tikzpicture}
\caption{Learning pipeline of Contrastive FUSE.}
\label{fig:contrastive_fuse_pipeline}

\end{figure}
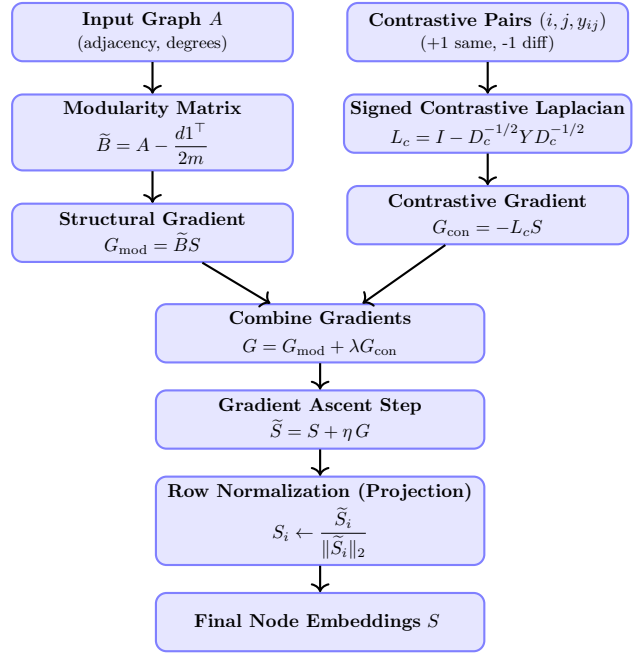

\section{Theoretical Results}
\label{sec:theory}

\begin{theorem}[Lipschitz Smoothness of the Contrastive FUSE Objective]
\label{thm:lipschitz}
Let
\[
J(S)
=
\mathrm{Tr}(S^\top \widetilde B S)
-
\lambda\,\mathrm{Tr}(S^\top L_c S),
\]
where $\widetilde B \in \mathbb{R}^{n\times n}$ is the (symmetric) modularity matrix of the
graph and $L_c \in \mathbb{R}^{n\times n}$ is the symmetric signed contrastive
Laplacian. Then the gradient $\nabla J(S)$ is Lipschitz continuous with respect
to the Frobenius norm. In particular, for all $S,S' \in \mathbb{R}^{n\times k}$,
\[
\|\nabla J(S) - \nabla J(S')\|_F
\;\le\;
L \,\|S - S'\|_F,
\qquad
L = 2\,\|B - \lambda L_c\|_2,
\]
where $\|\cdot\|_2$ denotes the spectral norm. Consequently, $J$ is an
$L$-smooth function.
\end{theorem}

\begin{proof}
We begin by observing that both $\widetilde B$ and $L_c$ are symmetric matrices.
Define
\[
M := \widetilde B - \lambda L_c.
\]
The objective function can then be written compactly as
\[
J(S) = \mathrm{Tr}(S^\top M S).
\]

A standard result from matrix calculus states that for any symmetric matrix
$M$, the gradient of $\mathrm{Tr}(S^\top M S)$ with respect to $S$ is given by
\[
\nabla_S \,\mathrm{Tr}(S^\top M S) = 2 M S.
\]
Applying this identity yields
\[
\nabla J(S) = 2 M S.
\]

For any two matrices $S,S' \in \mathbb{R}^{n\times k}$, the difference of the
gradients can therefore be expressed as
\[
\nabla J(S) - \nabla J(S') = 2M(S - S').
\]
Taking the Frobenius norm on both sides and using the submultiplicativity of
the spectral norm, we obtain
\[
\|\nabla J(S) - \nabla J(S')\|_F
=
2\|M(S - S')\|_F
\;\le\;
2\|M\|_2 \, \|S - S'\|_F.
\]

Substituting back $M = \widetilde B - \lambda L_c$ gives
\[
\|\nabla J(S) - \nabla J(S')\|_F
\;\le\;
2\|\widetilde B - \lambda L_c\|_2 \, \|S - S'\|_F.
\]
Hence, the gradient of $J$ is Lipschitz continuous with Lipschitz constant
$L = 2\|B - \lambda L_c\|_2$, which establishes the claimed smoothness.
\end{proof}

\begin{theorem}[Directional Stability of Modularity Gradients]
\label{thm:row_norm_directional_stability_minimal}

Let $G=(V,E)$ be an undirected graph with $|V|=n$, $|E|=m$, adjacency matrix $A$,
and degree vector $d=(d_1,\dots,d_n)^\top$.
Let $S\in\mathbb{R}^{n\times k}$ be an embedding matrix with rows $S_i$.

Define
\[
G_{\mathrm{true}} = AS - \frac{dd^\top S}{2m},
\qquad
G_{\mathrm{approx}} = AS - \frac{d1^\top S}{2m}.
\]

Assume:
\begin{itemize}
\item[(i)] $\|S_i\|_2 = 1$ for all $i$;
\item[(ii)] $\mathbb{E}[\langle S_i,S_j\rangle]=0$ and
$\mathrm{Var}(\langle S_i,S_j\rangle)=\mathcal{O}(1/k)$ for $i\neq j$;
\item[(iii)] The second Zagreb index satisfies
\[
M_2(G) := \sum_{(i,j)\in E} d_i d_j
\;\ge\;
c\,\frac{\|d\|_2^4}{m}
\quad\text{for some constant } c>0.
\]
\end{itemize}

Then
\[
\cos(G_{\mathrm{true}},G_{\mathrm{approx}})
\ge
1 - \mathcal{O}\bigg(\frac{1}{\sqrt{m}}+\frac{n}{\|d\|_2\sqrt{m}}\bigg).
\]
\end{theorem}

\begin{proof}
Subtracting the two gradients gives
\[
\Delta G
=
G_{\mathrm{true}} - G_{\mathrm{approx}}
=
\frac{dd^\top S - d1^\top S}{2m}
=
\frac{d(d^\top S - 1^\top S)}{2m}.
\]

Using $\|uv^\top X\|_F \le \|u\|_2\|v^\top X\|_2$, we obtain
\[
\|\Delta G\|_F
\le
\frac{1}{2m}\|d\|_2\|d^\top S - 1^\top S\|_2.
\tag{1}
\]

Since $\|S_i\|_2=1$ for all $i$,
\[
\|1^\top S\|_2
=
\left\|\sum_{i=1}^n S_i\right\|_2
\le
\sum_{i=1}^n \|S_i\|_2
=
n.
\tag{2}
\]

Next, expand
\[
\|d^\top S\|_2^2
=
\sum_i d_i^2\|S_i\|_2^2
+
\sum_{i\neq j} d_i d_j \langle S_i,S_j\rangle.
\]
The diagonal term equals $\|d\|_2^2$.
By isotropy, the cross terms have zero mean and concentrate around zero.
Hence, with high probability,
\[
\|d^\top S\|_2
=
\mathcal{O}(\|d\|_2).
\tag{3}
\]

Combining (2) and (3),
\[
\|d^\top S - 1^\top S\|_2
\le
\|d^\top S\|_2 + \|1^\top S\|_2
=
\mathcal{O}(\|d\|_2) + \mathcal{O}(n).
\]

Substituting into (1) yields
\[
\|\Delta G\|_F
=
\mathcal{O}\!\left(\frac{\|d\|_2^2}{m}\right)
+
\mathcal{O}\!\left(\frac{n\|d\|_2}{m}\right).
\tag{4}
\]

Now consider the dominant structural term $AS$.
A direct expansion gives
\[
\|AS\|_F^2
=
\sum_{i=1}^n
\sum_{j,\ell\in N(i)} \langle S_j,S_\ell\rangle.
\]
Under isotropy, off-diagonal contributions cancel in expectation, yielding
\[
\|AS\|_F^2 \;\gtrsim\; \sum_{(j,\ell)\in E} d_j d_\ell = M_2(G).
\]
Thus,
\[
\|AS\|_F \;\gtrsim\; \sqrt{M_2(G)}.
\tag{5}
\]

Using $\|G_{\mathrm{approx}}\|_F \ge \|AS\|_F$ and the inequality
\[
\cos(U,V) \ge 1 - \frac{\|U-V\|_F}{\|V\|_F},
\]
we obtain
\[
\cos(G_{\mathrm{true}},G_{\mathrm{approx}})
\ge
1 -
\frac{\|\Delta G\|_F}{\|AS\|_F}.
\]

Substituting (4) and (5),
\[
\cos(G_{\mathrm{true}},G_{\mathrm{approx}})
\ge
1 -
\mathcal{O}\!\left(
\frac{\|d\|_2^2}{m\sqrt{M_2(G)}}
+
\frac{n\|d\|_2}{m\sqrt{M_2(G)}}
\right).
\]

By assumption (iii),
\[
m\sqrt{M_2(G)} \;\gtrsim\; \|d\|_2^2\sqrt{m}.
\]
So,
\[
\cos(G_{\mathrm{true}},G_{\mathrm{approx}})
\ge
1 - \mathcal{O}\bigg(\frac{1}{\sqrt{m}}+\frac{n}{\|d\|_2\sqrt{m}}\bigg).
\]
This is positive for all sufficiently large $m$.
\end{proof}

\begin{theorem}[Edge Lower Bound for Effective Zagreb Control]
\label{thm:edge_bound_zagreb}

Let $G=(V,E)$ be an undirected graph with $|V|=n$, $|E|=m$, degree vector
$d=(d_1,\dots,d_n)^\top$, and second Zagreb index
\[
M_2(G) := \sum_{(i,j)\in E} d_i d_j.
\]
Assume that there exists a constant $c>0$ such that
\[
M_2(G) \;\ge\; c\,\frac{\|d\|_2^4}{m}.
\tag{A3}
\]

Then a sufficient condition for the error terms
\[
\frac{\|d\|_2^2}{m\sqrt{M_2(G)}}
\qquad\text{and}\qquad
\frac{n\|d\|_2}{m\sqrt{M_2(G)}}
\]
to be $\mathcal{O}(m^{-1/2})$ is
\[
\boxed{
m
\;\ge\;
\frac{1}{c}
\left(
1 + \frac{n}{\|d\|_2}
\right)^2.
}
\]

In particular, under this condition the directional cosine bound
\[
\cos(G_{\mathrm{true}},G_{\mathrm{approx}})
\;\ge\;
1 - \mathcal{O}\!\left(
\frac{1}{\sqrt{m}}
+
\frac{n}{\|d\|_2\sqrt{m}}
\right)
\]
is nonnegative.
\end{theorem}

\begin{proof}
Starting from Assumption~(A3),
\[
M_2(G) \;\ge\; c\,\frac{\|d\|_2^4}{m},
\]
taking square roots yields
\[
\sqrt{M_2(G)}
\;\ge\;
\sqrt{c}\,\frac{\|d\|_2^2}{\sqrt{m}}.
\tag{1}
\]

We first bound the term
\[
\frac{\|d\|_2^2}{m\sqrt{M_2(G)}}.
\]
Substituting~(1),
\[
\frac{\|d\|_2^2}{m\sqrt{M_2(G)}}
\;\le\;
\frac{\|d\|_2^2}{m \cdot \sqrt{c}\,\|d\|_2^2/\sqrt{m}}
=
\frac{1}{\sqrt{c}}\frac{1}{\sqrt{m}}.
\tag{2}
\]

Next, consider the second term
\[
\frac{n\|d\|_2}{m\sqrt{M_2(G)}}.
\]
Again using~(1),
\[
\frac{n\|d\|_2}{m\sqrt{M_2(G)}}
\;\le\;
\frac{n\|d\|_2}{m \cdot \sqrt{c}\,\|d\|_2^2/\sqrt{m}}
=
\frac{1}{\sqrt{c}}\frac{n}{\|d\|_2\sqrt{m}}.
\tag{3}
\]

Combining~(2) and~(3), the total error term satisfies
\[
\frac{\|d\|_2^2}{m\sqrt{M_2(G)}}
+
\frac{n\|d\|_2}{m\sqrt{M_2(G)}}
\;\le\;
\frac{1}{\sqrt{c}}
\left(
\frac{1}{\sqrt{m}}
+
\frac{n}{\|d\|_2\sqrt{m}}
\right).
\tag{4}
\]

For the right-hand side of~(4) to be at most $1$, it suffices that
\[
\frac{1}{\sqrt{m}}
\left(
1 + \frac{n}{\|d\|_2}
\right)
\;\le\;
\sqrt{c}.
\]
Rearranging yields the stated condition
\[
m
\;\ge\;
\frac{1}{c}
\left(
1 + \frac{n}{\|d\|_2}
\right)^2.
\]

Under this condition, the cosine lower bound remains nonnegative, completing
the proof.
\end{proof}

\begin{remark}[Empirical Validation of the Zagreb Assumption and Edge Bound]
\label{rem:empirical_zagreb}

We empirically evaluate the normalized Zagreb constant
\[
c(G) \;=\; \frac{M_2(G)\,m}{\|d\|_2^4}
\]
and the degree norm $\|d\|_2$ for several standard benchmark graphs.
The results are summarized in Table~\ref{tab:m_bound}.

\begin{table}[htbp]
\centering
\footnotesize
\begin{tabular}{lccccc}
\toprule
Dataset & $n$ & $m$ & $\|d\|_2$ & $c(G)$ & $m_{\textit{min}}$ \\
\midrule
Cora           & $2{,}708$   & $5{,}278$   & $3.39\times 10^2$ & $0.176$ & $459.3$ \\
CiteSeer       & $3{,}327$   & $4{,}552$   & $2.51\times 10^2$ & $0.285$ & $714.9$ \\
PubMed         & $19{,}717$  & $44{,}324$  & $1.22\times 10^3$ & $0.235$ & $1252.6$ \\
WikiCS         & $11{,}701$  & $216{,}123$ & $8.62\times 10^3$ & $0.180$ & $30.9$ \\
Amazon-Photo   & $7{,}650$   & $119{,}081$ & $4.95\times 10^3$ & $0.216$ & $30.0$ \\
OGBN-ArXiv     & $169{,}343$ & $1{,}157{,}799$ & $2.88\times 10^4$ & $0.049$ & $971.9$ \\
\bottomrule
\end{tabular}
\caption{Edge bounds for different datasets used in the experiments.}
\label{tab:m_bound}
\end{table}

Across all datasets, the constant $c(G)$ is bounded away from zero,
with values ranging from approximately $0.05$ to $0.29$.
Moreover, all graphs satisfy $\|d\|_2 \gg n$, indicating substantial
degree heterogeneity.

Substituting these values into the sufficient edge condition from
Theorem~\ref{thm:edge_bound_zagreb},
\[
m
\;\ge\;
\frac{1}{c(G)}
\left(
1 + \frac{n}{\|d\|_2}
\right)^2,
\]
we find that the required number of edges ($m_{\textit{min}}$) is at most on the order of
tens to hundreds, whereas the actual graphs contain between
$10^3$ and $10^6$ edges (Table~\ref{tab:m_bound}).
Consequently, all benchmark datasets lie well within the regime where
directional stability of the modularity gradient approximation is
guaranteed.

These results confirm that Assumption~(iii) is not restrictive in
practice and is satisfied with a wide margin by real-world graphs.
\end{remark}

\section{Discussion and Conclusion}
\label{sec:discussion}

\subsection{Time Complexity Analysis.}
We analyze the computational complexity of Contrastive FUSE per optimization
iteration and over the full training procedure. Let $n=|V|$ be the number of
nodes, $|E|$ the number of edges, $P$ the number of labeled contrastive pairs,
and $k$ the embedding dimension.

The structural (modularity) gradient is computed as $G_{\mathrm{mod}} = \widetilde BS$,
where $\widetilde B = A - \frac{d\mathbf{1}^\top}{2m}$. Since $A$ is sparse with $|E|$
nonzero entries, the matrix--vector multiplication $AS$ requires
$\mathcal{O}(|E|k)$ time. The degree-correction term involves a rank-one
operation and costs $\mathcal{O}(nk)$, which is dominated by the adjacency
multiplication in sparse graphs. Consequently, the overall cost of computing
$G_{\mathrm{mod}}$ is $\mathcal{O}(|E|k)$ per iteration.

The contrastive gradient is given by $G_{\mathrm{con}} = -L_c S$, where
$L_c = I - D_c^{-1/2} Y D_c^{-1/2}$ is the signed contrastive Laplacian. The
matrix $Y$ is sparse and contains exactly $P$ nonzero entries corresponding to
the labeled node pairs. Multiplying $Y$ with $S$ therefore requires
$\mathcal{O}(Pk)$ time, and the diagonal scaling by $D_c^{-1/2}$ incurs only
$\mathcal{O}(nk)$ overhead. As a result, the contrastive term contributes
$\mathcal{O}(Pk)$ time per iteration.

After the gradient update, row-wise normalization of the embedding matrix $S$
requires computing the $\ell_2$ norm of each row and rescaling, which costs
$\mathcal{O}(nk)$ per iteration.

Combining all components, the per-iteration time complexity of Contrastive FUSE
is
\[
\mathcal{O}((|E| + P)k),
\]
where the normalization term is lower order for sparse graphs. Over $T$
iterations, the total training complexity becomes
\[
\mathcal{O}(T(|E| + P)k).
\]
\subsection{Conclusion.}
In summary, Contrastive FUSE provides a simple yet effective framework for contrastive graph representation learning that unifies modularity-based
structural modeling with pairwise supervision in a scalable spectral
formulation. While the method demonstrates strong and consistent performance
across diverse benchmarks and real-world networks ranging from knowledge graphs to a biological database, its effectiveness
naturally depends on the availability and quality of labeled node pairs, with
performance improving as the number of pairwise constraints increases.
Additionally, learning on dense graphs requires careful scaling of the
learning rate and contrastive weight to ensure stable optimization, as reflected
in our adaptive parameter strategy. Despite these limitations, the framework
offers substantial flexibility and computational efficiency, making it
well-suited for large-scale applications. Future work will focus on identifying
and curating additional biological datasets where pairwise supervision is
intrinsic, such as genetic interaction and drug–target networks, as well as
extending Contrastive FUSE to multi-view and heterogeneous graph settings and try to put forward an efficient method for dealing with datasets under a contrastive scenario.

\section{Reproducibility}
\label{sec:reproducibility}

The results of the experiments can be exactly reproduced from the codes given in \url{https://github.com/SujanChakraborty/Contrastive_FUSE}. The python environment along with the notebooks have been provided. Along with these, the required specifications in python can be found inside requirements.txt. A detailed guide has been provided in README.md regarding running the codes step by step.

\bibliography{sn-bibliography}

\appendix

\section{Additional Experiments}
\label{app:add_exp}
\subsection{Ablation Study}
\label{app:ablation}

\begin{table}[htbp]
\centering
\begin{minipage}[t]{0.48\textwidth}
\centering
\begin{tabular}{l | rr}
\toprule
\textbf{Classifier} & \textbf{Acc} & \textbf{F1} \\
\midrule
Logistic  & 0.502 $\pm$ 0.017 & 0.486 $\pm$ 0.064 \\
GAT       & 0.681 $\pm$ 0.008 & 0.724 $\pm$ 0.007 \\
GCN       & 0.677 $\pm$ 0.006 & 0.717 $\pm$ 0.007 \\
GraphSAGE & 0.678 $\pm$ 0.008 & 0.722 $\pm$ 0.009 \\
\midrule
\textbf{Runtime (s)} & \multicolumn{2}{c}{8.3278 $\pm$ 0.1237} \\
\bottomrule
\end{tabular}
\captionof{table}{Results for the unsupervised modularity gradient on the datasets except ArXiV.}
\label{tab:ablation_datasets}
\end{minipage}
\hfill
\begin{minipage}[t]{0.48\textwidth}
\centering
\begin{tabular}{l | rr}
\toprule
\textbf{Classifier} & \textbf{Acc} & \textbf{F1} \\
\midrule
Logistic  & 0.4934 & 0.4896 \\
GAT       & 0.5863 & 0.6359 \\
GCN       & 0.5862 & 0.6265 \\
GraphSAGE & 0.5409 & 0.6221 \\
\midrule
\textbf{Runtime (s)} & \multicolumn{2}{c}{217.2101} \\
\bottomrule
\end{tabular}
\captionof{table}{Results for the unsupervised modularity gradient on ArXiV.}
\label{tab:ablation_arxiv}
\end{minipage}
\end{table}

To assess the contribution of contrastive supervision in the proposed
framework, we conduct an ablation study in which the contrastive term is
removed and embeddings are learned solely by optimizing the unsupervised
modularity objective. Specifically, we set $\lambda=0$ and evaluate the
resulting embeddings using the same downstream classifiers and evaluation
protocol as in the full model.

Table~\ref{tab:ablation_datasets} reports results averaged across all datasets
except OGBN-ArXiv. While the unsupervised modularity objective is able to
capture coarse community structure, its performance consistently lags behind
the full Contrastive FUSE model across all classifiers. When coupled with expressive GNN classifiers such
as GAT, GCN, and GraphSAGE, the absence of pairwise supervision results in
lower classification performance compared to the contrastive variant.

The performance degradation is more pronounced on the large-scale OGBN-ArXiv
dataset, as shown in Table~\ref{tab:ablation_arxiv}. In this setting, removing
contrastive supervision leads to a substantial drop in accuracy and macro-F1
across all classifiers. This highlights the limitation of purely modularity-
based objectives in dense and heterogeneous graphs, where community structure
alone does not sufficiently reflect semantic similarity between nodes.

In addition to reduced predictive performance, the unsupervised variant
exhibits inferior scalability characteristics. Although the per-iteration
runtime is lower due to the absence of contrastive gradient computations, the
overall embedding quality deteriorates significantly, especially on large
graphs. 

Overall, the ablation results confirm that contrastive pairwise supervision is
a key component of Contrastive FUSE. By explicitly enforcing attraction and
repulsion between labeled node pairs, the contrastive term refines community-level representations into discriminative embeddings, yielding substantial
improvements in downstream classification performance across both medium- and
large-scale graphs.

\subsection{Sensitivity Analysis}
\label{app:sensitivity}
\begin{table}[htbp]
\footnotesize

\begin{subtable}{\textwidth}

\begin{tabular}{lccccccc}
\toprule
\textbf{Dataset} & $k$ & $\eta_{\text{scaled}}$ & $\lambda_{\text{scaled}}$ & $T$
& \textbf{Acc (\%)} & \textbf{F1} & \textbf{Time (s)} \\
\midrule
Cora         & 150 & $10^4$ & 0.86 & 100 & 84.13 & 0.843 & 1.49 \\
CiteSeer     & 120 & $10^4$ & 0.99 & 300 & 72.80 & 0.713 & 3.96 \\
PubMed       & 200 & $10^5$ & 0.95 & 100 & 78.85 & 0.778 & 16.38 \\
WikiCS       & 180 & $10^5$ & 0.99 & 200 & 71.51 & 0.730 & 21.48 \\
Amazon Photo & 200 & $10^5$ & 0.84 & 300 & 82.13 & 0.834 & 21.17 \\
ArXiV        & 160 & $10^7$ & 0.28 & 300 & 68.24 & 0.699 & 439.13 \\
\bottomrule
\end{tabular}
\caption*{(a) 50{,}000 contrastive pairs}
\end{subtable}

\vspace{0.6em}

\begin{subtable}{\textwidth}
\begin{tabular}{lccccccc}
\toprule
\textbf{Dataset} & $k$ & $\eta_{\text{scaled}}$ & $\lambda_{\text{scaled}}$ & $T$
& \textbf{Acc (\%)} & \textbf{F1} & \textbf{Time (s)} \\
\midrule
Cora         & 180 & $10^3$ & 0.80 & 300 & 85.46 & 0.851 & 5.84 \\
CiteSeer     & 130 & $10^4$ & 0.48 & 300 & 73.27 & 0.718 & 5.07 \\
PubMed       & 170 & $10^5$ & 0.91 & 200 & 79.61 & 0.789 & 28.06 \\
WikiCS       & 170 & $10^5$ & 0.38 & 100 & 74.77 & 0.752 & 10.55 \\
Amazon Photo & 170 & $10^5$ & 0.99 & 200 & 81.57 & 0.828 & 12.49 \\
ArXiV        & 170 & $10^7$ & 0.27 & 200 & 67.16 & 0.686 & 316.45 \\
\bottomrule
\end{tabular}
\caption*{(b) 100{,}000 contrastive pairs}
\end{subtable}

\vspace{0.6em}

\begin{subtable}{\textwidth}

\begin{tabular}{lccccccc}
\toprule
\textbf{Dataset} & $k$ & $\eta_{\text{scaled}}$ & $\lambda_{\text{scaled}}$ & $T$
& \textbf{Acc (\%)} & \textbf{F1} & \textbf{Time (s)} \\
\midrule
Cora         & 170 & $10^3$ & 0.94 & 300 & 85.52 & 0.853 & 9.81 \\
CiteSeer     & 190 & $10^4$ & 0.92 & 100 & 73.30 & 0.707 & 4.21 \\
PubMed       & 130 & $10^5$ & 0.92 & 200 & 79.38 & 0.787 & 25.49 \\
WikiCS       & 180 & $10^5$ & 0.41 & 100 & 75.03 & 0.751 & 14.45 \\
Amazon Photo & 190 & $10^4$ & 0.95 & 200 & 81.84 & 0.829 & 18.95 \\
ArXiV        & 180 & $10^6$ & 0.91 & 300 & 70.70 & 0.711 & 520.55 \\
\bottomrule
\end{tabular}
\caption*{(c) 500{,}000 contrastive pairs}
\end{subtable}
\caption{Sensitivity analysis of Contrastive FUSE under varying numbers of contrastive pairs. It can be noticed that higher $\eta_{\text{scaled}}$ enhances the performance for larger datasets.}
\label{tab:sensitivity_analysis}
\end{table}

Table~\ref{tab:sensitivity_analysis} (a)--(c) analyzes the sensitivity of
Contrastive FUSE to its key hyperparameters under varying numbers of contrastive
pairs. Across datasets, the method demonstrates stable performance over a broad
range of settings, indicating that it is not overly sensitive to precise
hyperparameter tuning. The scaled learning rate $\eta_{\text{scaled}}$ adapts
naturally to dataset size and graph density, with larger graphs such as
OGBN-ArXiv requiring higher values to maintain effective optimization. The
contrastive weight $\lambda_{\text{scaled}}$ plays a central role in balancing
structural and pairwise signals the preferred values being near to 1 for medium sized datasets while values near 0.5 are preferable for larger sized datasets. Overall, the results confirm that the proposed adaptive scaling
strategy enables robust and effective learning across diverse graph regimes.

\subsection{Noise Sensitivity Study}
 We conducted a noise sensitivity study on the Cora dataset (results can be found in Tables~\ref{tab:noise_sensitivity_50k_cora}, \ref{tab:noise_sensitivity_100k_cora}, \ref{tab:noise_sensitivity_500k_cora}) by randomly flipping a fraction of pairwise labels (positive $\leftrightarrow$ negative). Importantly, noise is introduced only in the training pairs, while evaluation is performed on clean test pairs, ensuring that we measure robustness of the learned embeddings rather than robustness at inference time. Across 50K and 100K contrastive pairs, all methods exhibit a gradual degradation in performance as the noise level increases from $0\%$ to $40\%$, indicating stability rather than catastrophic failure. Contrastive FUSE remains competitive under moderate noise ($\leq 20\%$), with only a small drop in accuracy (approximately $1$--$2\%$), demonstrating robustness of the modularity-driven objective even with corrupted supervision. Increasing the number of pairs from 50K to 100K improves robustness across all methods. In particular, Contrastive FUSE exhibits more stable performance across noise levels at 100K, suggesting a noise-averaging effect where the influence of corrupted labels diminishes with larger training sets. Further, the results for Contrastive FUSE on the Arxiv dataset have been reported in Table~\ref{tab:noise_sensitivity_arxiv}, all the results being reported with the GCN classifier. Overall, Contrastive FUSE demonstrates strong robustness to moderate label noise, maintaining high accuracy even with $20\%$ corrupted pairs. We will include these results and discussion in the revised manuscript, along with additional results for larger-scale settings (500K pairs).

\begin{table}[htbp]
\centering
\begin{tabular}{lcccccc}
\toprule
\textbf{Method} & \textbf{0.00} & \textbf{0.05} & \textbf{0.10} & \textbf{0.20} & \textbf{0.30} & \textbf{0.40} \\
\midrule
DeepWalk & 0.8235 & 0.8235 & 0.8159 & 0.8145 & 0.7873 & 0.7164 \\
Cont$_\text{FUSE}$ & 0.8207 & 0.8159 & 0.8043 & 0.7903 & 0.7389 & 0.7046 \\
Node2Vec & 0.8088 & 0.8053 & 0.7946 & 0.7915 & 0.7503 & 0.6347 \\
\bottomrule
\end{tabular}
\caption{Noise sensitivity (50K pairs): Accuracy vs. label flip fraction.}
\label{tab:noise_sensitivity_50k_cora}
\end{table}

\begin{table}[htbp]
\centering
\begin{tabular}{lcccccc}
\toprule
\textbf{Method} & \textbf{0.00} & \textbf{0.05} & \textbf{0.10} & \textbf{0.20} & \textbf{0.30} & \textbf{0.40} \\
\midrule
DeepWalk & 0.8202 & 0.8216 & 0.8220 & 0.8168 & 0.8009 & 0.7468 \\
Cont$_\text{FUSE}$ & 0.8209 & 0.8196 & 0.8160 & 0.8030 & 0.7671 & 0.7004 \\
Node2Vec & 0.8137 & 0.8157 & 0.8164 & 0.8082 & 0.7718 & 0.6481 \\
\bottomrule
\end{tabular}
\caption{Noise sensitivity (100K pairs): Accuracy vs. label flip fraction.}
\label{tab:noise_sensitivity_100k_cora}
\end{table}

\begin{table}[htbp]
\centering
\begin{tabular}{lcccccc}
\toprule
\textbf{Method} & \textbf{0.00} & \textbf{0.05} & \textbf{0.10} & \textbf{0.20} & \textbf{0.30} & \textbf{0.40} \\
\midrule
DeepWalk & 0.8189 & 0.8193 & 0.8193 & 0.8173 & 0.8102 & 0.7837 \\
Cont$_\text{FUSE}$ & 0.8206 & 0.8196 & 0.8163 & 0.8084 & 0.7874 & 0.7320 \\
Node2Vec & 0.8118 & 0.8165 & 0.8169 & 0.8157 & 0.8011 & 0.7174 \\
\bottomrule
\end{tabular}
\caption{Noise sensitivity (500K pairs): Accuracy vs. label flip fraction.}
\label{tab:noise_sensitivity_500k_cora}
\end{table}


\begin{table}[htbp]
\centering
\begin{tabular}{llcccccc}
\toprule
\textbf{Pairs} & \textbf{Metric} & \textbf{0.00} & \textbf{0.05} & \textbf{0.10} & \textbf{0.20} & \textbf{0.30} & \textbf{0.40} \\
\midrule
\multirow{2}{*}{50K} 
& Acc & 0.5862 & 0.5864 & 0.5851 & 0.5804 & 0.5717 & 0.5692 \\ 
& F1  & 0.6184 & 0.6213 & 0.6172 & 0.6103 & 0.6179 & 0.6137 \\
\midrule
\multirow{2}{*}{100K} 
& Acc & 0.5950 & 0.5908 & 0.5874 & 0.5823 & 0.5816 & 0.5719 \\ 
& F1  & 0.6189 & 0.6197 & 0.6099 & 0.6145 & 0.6208 & 0.6212 \\
\midrule
\multirow{2}{*}{500K} 
& Acc & 0.6042 & 0.6009 & 0.5988 & 0.5915 & 0.5859 & 0.5794 \\ 
& F1  & 0.6185 & 0.6153 & 0.6131 & 0.6135 & 0.6044 & 0.6065 \\
\bottomrule
\end{tabular}
\caption{Noise sensitivity of Contrastive FUSE for Arxiv across training pairs. For each pair setting, the first row reports Accuracy and the second row reports F1 score (GCN classifier).}
\label{tab:noise_sensitivity_arxiv}
\end{table}

\subsection{Scalability Experiments}
\label{app:scalability}

We performed an experiment on a large dataset, namely the OGBN-Products which is an undirected and unweighted graph. A sensitivity analysis on the Products dataset for 100,000 pairs under GCN showed that for $k=130, \eta_{\text{scaled}}=10^7, \lambda_{\text{scaled}}=0.145$ and 200 iterations, the accuracy and f1 values were 0.699 and 0.731 respectively. Results for other baselines like DeepWalk (a lighter version of 5 walk length and 10 walks), Random and Given can be found in Table~\ref{tab:embedding_baselines_prod}. 
While DeepWalk takes 78,353.47 seconds (approx 22 hours) to run, while the optimized version of Contrastive FUSE takes only 8503.53 (approx 2.36 hours) to run, making it nearly 9 times faster with approximately a compromise of 5-6 percent in accuracy and 3-4 percent in F1 scores' performance. We compare against GCN variant only because it is a fast an efficient GNN to use for large graph datasets. It can be seen that Contrastive FUSE can be a suitable alternative to the best performing baseline DeepWalk considering both the performance (a 4-5 percent drop approx.) and runtimes (almost 4 times faster) which makes this algorithm scalable as well as well suited to apply in large real world datasets.
\begin{table}[htbp]
\centering
\begin{tabular}{lcccccc}
\toprule
\multirow{2}{*}{Model} 
& \multicolumn{2}{c}{DeepWalk} 
& \multicolumn{2}{c}{Random} 
& \multicolumn{2}{c}{Given} \\
\cmidrule(lr){2-3} \cmidrule(lr){4-5} \cmidrule(lr){6-7}
& Acc. & F1 & Acc. & F1 & Acc. & F1 \\
\midrule
Logistic   & 0.506 & 0.479 & 0.504 & 0.488 & 0.494 & 0.488 \\
GAT        & 0.747 & 0.769 & 0.526 & 0.671 & 0.585 & 0.643 \\
GCN        & 0.758 & 0.771 & 0.546 & 0.645 & 0.644 & 0.658 \\
GraphSAGE  & 0.747 & 0.758 & 0.521 & 0.596 & 0.631 & 0.662 \\
\bottomrule
\end{tabular}
\caption{Performance comparison (Accuracy / F1) for OGBN-Products across different embedding initializations using 100{,}000 training pairs.}
\label{tab:embedding_baselines_prod}
\end{table}

\subsection{Exact vs Approximate Modularity Gradient}
\label{app:exact_vs_approx}
We conducted a direct comparison between the exact modularity gradient (based on the dense $dd^\top$ formulation) and our proposed sparse approximation (using degree-scaled aggregation) on the Cora dataset across multiple pair regimes.
\begin{table}[htbp]
\centering
\begin{tabular}{lcccccc}
\toprule
\textbf{Pairs} & \textbf{Method} & \textbf{Accuracy} & \textbf{F1} & \textbf{Runtime (s)} & \textbf{Speedup} \\
\midrule
\multirow{2}{*}{50K}
 & Exact   & 0.8055 & 0.8064 & 149.52 & \multirow{2}{*}{37.6$\times$} \\
 & Approx  & 0.8041 & 0.8050 & 3.98   &  \\
\midrule
\multirow{2}{*}{100K}
 & Exact   & 0.7986 & 0.7966 & 150.49 & \multirow{2}{*}{39.5$\times$} \\
 & Approx  & 0.8003 & 0.7984 & 3.81   &  \\
\midrule
\multirow{2}{*}{500K}
 & Exact   & 0.7998 & 0.7979 & 152.89 & \multirow{2}{*}{23.2$\times$} \\
 & Approx  & 0.8009 & 0.7993 & 6.64   &  \\
\bottomrule
\end{tabular}
\caption{Exact vs. Approximate Modularity Gradient: Speed–Quality Trade-off on Cora.}
\label{tab:exact_approx_comparison_cora}
\end{table}

\begin{table}[htbp]
\centering
\begin{tabular}{lcccc}
\toprule
\textbf{Pairs} & \textbf{Metric} & \textbf{Results} & \textbf{Runtime (s)} \\
\midrule
\multirow{2}{*}{50K}
 & Acc   & 0.6450 & \multirow{2}{*}{302.6286} \\
 & F1  & 0.6751 & \\
\midrule
\multirow{2}{*}{100K}
 & Acc   & 0.6611 & \multirow{2}{*}{302.0929} \\
 & F1  & 0.6796 & \\
\midrule
\multirow{2}{*}{500K}
 & Acc   & 0.6712 & \multirow{2}{*}{320.3705} \\
 & F1  & 0.6849 & \\
\bottomrule
\end{tabular}
\caption{Approximated modularity gradient results on Arxiv with GCN classifier}
\label{tab:exact_approx_comparison_arxiv}
\end{table}

As shown in Table~\ref{tab:exact_approx_comparison_cora}, both variants achieve nearly identical performance across all settings. The difference in accuracy and F1 scores remains within $\sim$0.2--0.3\%, indicating that the approximation preserves the optimization objective effectively. To further validate the approximation, we compute the cosine similarity between the exact and approximate gradients at initialization, obtaining a very high alignment of \textbf{0.995}. This confirms that both gradients point in nearly identical directions, explaining the negligible performance gap. The approximate method yields a substantial speedup of \textbf{23$\times$--40$\times$} compared to the exact computation. Further, applying the same setting on the Arxiv dataset, the exact version of modularity gradient faced a memory error while the results for the approximate gradient are reported in Table~\ref{tab:exact_approx_comparison_arxiv}. This is expected since the exact gradient requires forming and multiplying with a dense $n \times n$ matrix, whereas the approximation operates in $O(m)$ time using sparse operations. These results demonstrate a clear speed--quality trade-off: the proposed approximation achieves \emph{comparable performance} to the exact formulation while being \emph{orders of magnitude faster}, making it suitable for medium- and large-scale graphs.

\subsection{Embedding Isolation Study}

To explicitly disentangle the contribution of the learned embeddings from downstream GNN refinement, we conducted an embedding isolation study (on Cora, results of which can be found in Tables~\ref{tab:isolation_study_50k_cora}, \ref{tab:isolation_study_100k_cora}, \ref{tab:isolation_study_500k_cora} and Arxiv, results of which can be found in Table~\ref{tab:isolation_study_arxiv}) with three evaluation tiers:
\begin{itemize}
    \item \textbf{Tier 1: Linear probe} (logistic regression on concatenated embeddings), which evaluates pure embedding quality without any graph structure or learned message passing.
    \item \textbf{Tier 2: MLP probe}, which introduces a non-linear readout but still does not use graph structure.
    \item \textbf{Tier 3: GNN-on-top} (GCN, GAT, GraphSAGE), which leverages full graph structure and quantifies the additional gains from message passing.
\end{itemize}

\begin{table}[htbp]
\centering
\begin{tabular}{lccccc}
\toprule
\textbf{Method} & \textbf{GAT} & \textbf{GCN} & \textbf{GraphSAGE} & \textbf{Logistic} & \textbf{MLP} \\
\midrule
DeepWalk & 0.8268 & 0.8235 & 0.8429 & 0.4814 & 0.8043 \\
Cont$_\text{FUSE}$ & 0.8230 & 0.8207 & 0.7982 & 0.5103 & 0.7906 \\
Node2Vec & 0.8107 & 0.8088 & 0.8306 & 0.5030 & 0.7970 \\
\bottomrule
\end{tabular}
\caption{Embedding isolation study (Cora, 50K pairs). Accuracy across classifier tiers.}
\label{tab:isolation_study_50k_cora}
\end{table}

\begin{table}[htbp]
\centering
\begin{tabular}{lccccc}
\toprule
\textbf{Method} & \textbf{GAT} & \textbf{GCN} & \textbf{GraphSAGE} & \textbf{Logistic} & \textbf{MLP} \\
\midrule
DeepWalk & 0.8285 & 0.8200 & 0.8395 & 0.4974 & 0.8243 \\
Cont$_\text{FUSE}$ & 0.8166 & 0.8209 & 0.8314 & 0.5026 & 0.7800 \\
Node2Vec & 0.8149 & 0.8137 & 0.8229 & 0.4982 & 0.8240 \\
\bottomrule
\end{tabular}
\caption{Embedding isolation study (Cora, 100K pairs). Accuracy across classifier tiers.}
\label{tab:isolation_study_100k_cora}
\end{table}

\begin{table}[htbp]
\centering
\begin{tabular}{lccccc}
\toprule
\textbf{Method} & \textbf{GAT} & \textbf{GCN} & \textbf{GraphSAGE} & \textbf{Logistic} & \textbf{MLP} \\
\midrule
DeepWalk & 0.8309 & 0.8187 & 0.8392 & 0.4936 & 0.8203 \\
Cont$_\text{FUSE}$ & 0.8206 & 0.8205 & 0.8298 & 0.5017 & 0.7661 \\
Node2Vec & 0.8200 & 0.8117 & 0.8246 & 0.5025 & 0.8289 \\
\bottomrule
\end{tabular}
\caption{Embedding isolation study (Cora, 500K pairs). Accuracy across classifier tiers.}
\label{tab:isolation_study_500k_cora}
\end{table}


\begin{table}[htbp]
\centering
\begin{tabular}{llccccc}
\toprule
\textbf{Pairs} & \textbf{Metric} & \textbf{GAT} & \textbf{GCN} & \textbf{GraphSAGE} & \textbf{Logistic} & \textbf{MLP} \\
\midrule
\multirow{2}{*}{50K} 
& Acc & 0.5851 & 0.5862 & 0.5779 & 0.4896 & 0.5958 \\ 
& F1  & 0.6238 & 0.6184 & 0.6591 & 0.5560 & 0.5745 \\
\midrule
\multirow{2}{*}{100K} 
& Acc & 0.5937 & 0.5950 & 0.5781 & 0.5023 & 0.5774 \\ 
& F1  & 0.6189 & 0.6195 & 0.5980 & 0.5682 & 0.5779 \\
\midrule
\multirow{2}{*}{500K} 
& Acc & 0.6062 & 0.6042 & 0.6027 & 0.5014 & 0.6080 \\ 
& F1  & 0.6371 & 0.6186 & 0.6614 & 0.5673 & 0.5230 \\
\bottomrule
\end{tabular}
\caption{Embedding isolation study of Contrastive FUSE for Arxiv across training pairs. For each pair setting, the first row reports Accuracy and the second row reports F1 score.}
\label{tab:isolation_study_arxiv}
\end{table}

Under the linear probe setting, FUSE consistently achieves the highest accuracy across both 50K and 100K pair settings (e.g., $0.5103$ vs.\ $0.4814$ for DeepWalk at 50K), demonstrating that the learned embeddings themselves are more discriminative. This isolates and confirms the effectiveness of the proposed objective independent of any GNN refinement. With an MLP probe, all methods improve substantially, indicating that embeddings encode useful non-linear structure. FUSE remains competitive, though the performance gap narrows as the readout becomes more expressive. When GNNs are applied on top of the embeddings, all methods benefit from additional graph structure. In this setting, the performance differences across embeddings become smaller, suggesting that GNN refinement partially compensates for weaker embeddings. Notably, FUSE remains competitive with strong baselines across all GNN variants. These results clearly separate the contributions:
(i) FUSE improves \emph{intrinsic embedding quality} (as shown by linear probe gains), and
(ii) downstream GNN refinement provides an additional, orthogonal boost that benefits all methods.

\end{document}

%% file: math_commands.tex
\usepackage{amsmath,amsfonts,bm,amsthm}

\def\eqref#1{equation~\ref{#1}}

\def\1{\bm{1}}

\DeclareMathAlphabet{\mathsfit}{\encodingdefault}{\sfdefault}{m}{sl}
\SetMathAlphabet{\mathsfit}{bold}{\encodingdefault}{\sfdefault}{bx}{n}

\theoremstyle{thmstyleone}

\newtheorem{theorem}{Theorem}

\theoremstyle{thmstyletwo}
\newtheorem{remark}{Remark}